\documentclass[12pt]{article}

\usepackage[margin = 1.1in]{geometry}
\usepackage{amsmath,amssymb,amsthm}
\usepackage{authblk}
\usepackage{enumitem}
\usepackage{natbib}
\usepackage{hyperref}

\newtheorem{theorem}{Theorem}
\newtheorem{proposition}{Proposition}
\newtheorem{definition}{Definition}
\newtheorem{example}{Example}

\numberwithin{equation}{section}

\usepackage{bbm}
\usepackage{mathtools}

\newcommand*{\rbr}[1]{\left(#1\right)}
\newcommand*{\cbr}[1]{\left\{#1\right\}}
\newcommand*{\sbr}[1]{\left[#1\right]}
\newcommand*{\abr}[1]{\left|#1\right|}

\newcommand*{\trbr}[1]{(#1)}
\newcommand*{\tcbr}[1]{\{#1\}}

\DeclareMathOperator{\Binomial}{Binomial}
\DeclareMathOperator{\Uniform}{Uniform}

\newcommand*{\R}{\mathbb{R}}

\newcommand*{\X}{\mathbb{X}}
\newcommand*{\Y}{\mathbb{Y}}

\newcommand*{\eps}{\epsilon}

\newcommand*{\alg}{\mathcal{A}}
\newcommand*{\talgof}[1]{\mathcal{A}[#1]}
\newcommand*{\algof}[1]{\mathcal{A}\left[#1\right]}

\newcommand*{\dgp}{P}

\newcommand*{\dataset}{\mathcal{D}}

\newcommand*{\muh}{\widehat{\mu}}
\newcommand*{\mut}{\widetilde{\mu}}

\newcommand*{\algspace}{\mathfrak{A}}

\newcommand*{\datasetspace}{\mathfrak{D}}

\newcommand*{\modelspace}{\mathcal{M}_d}
\newcommand*{\modelclass}{\mathcal{M}_{d,*}}

\newcommand*{\test}{T}

\newcommand*{\Yhat}{\widehat{\mathcal{Y}}}
\newcommand*{\Ytilde}{\widetilde{\mathcal{Y}}}

\newcommand*{\Xcal}{\mathcal{X}}
\newcommand*{\Ycal}{\mathcal{Y}}

\newcommand*{\iidsim}{\stackrel{\textnormal{iid}}{\sim}}

\newcommand*{\tEE}[1]{\mathbb{E}[{#1}]}

\newcommand*{\EEst}[2]{\mathbb{E}\left[{#1}\ \middle| \ {#2}\right]}

\newcommand*{\tPP}[1]{\mathbb{P}\{{#1}\}}
\newcommand*{\PP}[1]{\mathbb{P}\left\{{#1}\right\}}
\newcommand*{\Pp}[2]{\mathbb{P}_{{#1}}\left\{{#2}\right\}}
\newcommand*{\PPst}[2]{\mathbb{P}\left\{{#1}\ \middle| \ {#2}\right\}}

\newcommand*{\tOne}[1]{{\mathbbm{1}}\{{#1}\}}
\newcommand*{\One}[1]{{\mathbbm{1}}\left\{{#1}\right\}}

\makeatletter
\newcommand*{\indep}{%
  \mathbin{%
    \mathpalette{\@indep}{}%
  }%
}
\newcommand*{\nindep}{%
  \mathbin{
    \mathpalette{\@indep}{\not}
  }%
}
\newcommand*{\@indep}[2]{%
  \sbox0{$#1\perp\m@th$}
  \sbox2{$#1=$}
  \sbox4{$#1\vcenter{}$}
  \rlap{\copy0}
  \dimen@=\dimexpr\ht2-\ht4-.2pt\relax
  \kern\dimen@
  {#2}%
  \kern\dimen@
  \copy0 
} 
\makeatother

\begin{document}
\title{Black-box tests for algorithmic stability}
\date{\today}
\author[1, 2]{Byol Kim\thanks{byolkim@uw.edu}}
\author[3]{Rina Foygel Barber\thanks{rina@uchicago.edu}}
\affil[1]{\footnotesize Department of Biostatistics, University of Washington, Seattle, WA, 98195, USA}
\affil[2]{\footnotesize eScience Institute, University of Washington, Seattle, WA, 98195, USA}
\affil[3]{\footnotesize Department of Statistics, University of Chicago, Chicago, IL, 60637, USA}

\maketitle

\begin{abstract}
Algorithmic stability is a concept from learning theory that expresses the degree to which changes to the input data (e.g., removal of a single data point) may affect the outputs of a regression algorithm. Knowing an algorithm's stability properties is often useful for many downstream applications---for example, stability is known to lead to desirable generalization properties and predictive inference guarantees. However, many modern algorithms currently used in practice are too complex for a theoretical analysis of their stability properties, and thus we can only attempt to establish these properties through an empirical exploration of the algorithm's behavior on various datasets. In this work, we lay out a formal statistical framework for this kind of {\em black-box testing} without any assumptions on the algorithm or the data distribution, and establish fundamental bounds on the ability of any black-box test to identify algorithmic stability.
\end{abstract}

\section{Introduction}

Consider an algorithm $\alg$ that maps datasets to fitted regression functions:
\[
	\alg : \ \underbrace{(X_1,Y_1), \dots, (X_n,Y_n)}_{\textnormal{Training data points in $\R^d\times\R$}} \ \ \mapsto \ \ \underbrace{\muh: \R^d\rightarrow \R}_{\textnormal{Fitted function}}.
\]
In many practical settings where regression algorithms are used, the distribution $\dgp$ of the data points $(X_i,Y_i)$ is potentially complex and cannot be assumed to satisfy strong simplifying conditions (such as a parametric model or a high degree of smoothness). Guarantees bounding the predictive error of the fitted model $\muh$ can still be obtained by testing its predictive accuracy on additional labeled data that was not used for training (i.e., a holdout set). Of course, in settings where the available training data is limited, we may not want to sacrifice sample size in order to obtain this independent holdout set---that is, if $n$ is the number of available labeled data points drawn from the unknown data distribution $\dgp$, we may want to run the regression algorithm $\alg$ on the entire set of $n$ data points. In this type of setting, where no holdout set is available, it is nonetheless possible to give guarantees on predictive error by assuming a condition of {\em stability} on the algorithm $\alg$ \citep{Bousquet2002Stability, Steinberger2020Conditional}. The stability condition essentially requires that slight perturbations of the training dataset $\big((X_i,Y_i)\big)_{i = 1,\dots,n}$, such as removing or replacing a single data point, can alter the fitted model $\muh$ only slightly. 

Unfortunately, verifying the stability of a given regression algorithm is itself challenging. Some algorithms are known to satisfy stability by construction---for example, if $\alg$ is a $k$-nearest neighbor algorithm, then removing a single data point in the training set will only perturb predictions locally near that point. However, in many settings, particularly with modern machine learning algorithms, it is not feasible to analyze an algorithm's stability properties theoretically, since the algorithm is too complex for this to be tractable. Instead, we can only examine its properties through empirical testing, to try to determine based on its output whether it appears to be stable or unstable. We may ask, therefore, whether stability of an algorithm can instead be validated empirically, to ensure that the algorithms used in practice also enjoy the theoretical guarantees that accompany the stability property.

In this work, our aim is to examine the question of whether it is possible to infer the stability of an algorithm through {\em black-box testing} that does not peer inside the construction of the algorithm but instead tries to determine the stability property by observing the empirical behavior of the algorithm on various datasets. In particular, we will establish fundamental bounds on our ability to test for the stability property in the absence of assumptions on the distribution $\dgp$ or the algorithm $\alg$.

\section{Background and framework} \label{sec:background}

We define an {\em algorithm} $\alg$ as any (possibly randomized) function taking any finite collection of $(x,y)$ pairs, where $x\in\R^d$ is a feature vector and $y\in\R$ is a response variable,\footnote{Throughout, we assume a feature space $\R^d$ and a response space $\R$ for clarity of the exposition, but our results hold more generally---we will comment more on this in Section~\ref{sec:discussion_uncountable}.} to a regression function $\muh$ mapping a feature vector $x\in\R^d$ to some predicted $y\in\R$. Formally, as a map between spaces, $\alg$ can be written as
\[
	\alg: \rbr{\bigcup_{n \geq 0} \rbr{\R^d \times \R}^n} \times [0,1] \to \modelspace,
\]
where $\modelspace = \{\textnormal{measurable functions}\ \mu: \R^d \to \R\}$. Given $n$ points $(x_1,y_1), \dots, (x_n,y_n)$, the fitted regression function $\muh$ obtained from $\alg$ is
\[
	\muh = \algof{(x_1,y_1), \dots, (x_n,y_n); \xi}, \quad \xi \sim \Uniform [0,1].
\]
Here, the role of the last argument $\xi$ is to encode the random behavior of $\alg$, if needed. For example, if $\alg$ uses stochastic gradient descent, $\xi$ determines the random sequence of data points seen by the algorithm. Effectively, we can think of $\xi$ as the random seed with which we initialize the procedure. On the other hand, if $\alg$ is a deterministic method, then the argument $\xi$ can simply be ignored. Thus, our notation unifies the deterministic and randomized algorithm settings.

We will assume a measurability condition on $\alg$---namely, we assume measurability of the function \[\big((x_1,y_1),\dots,(x_n,y_n),\xi,x \big)\mapsto \algof{(x_1,y_1), \dots, (x_n,y_n); \xi}(x),\] which runs $\alg$ on a dataset $(x_1,y_1),\dots,(x_n,y_n)$ with randomization term $\xi$ and then evaluates the resulting fitted function at $x$. We will also assume that $\alg$ is {\em symmetric}, i.e., given any training dataset $(x_1,y_1), \dots, (x_n,y_n)$ and any permutation $\sigma$ on $\{1,\dots,n\}$, it holds that
\[
	\algof{(x_1,y_1), \dots, (x_n,y_n); \xi} =\algof{(x_{\sigma(1)},y_{\sigma(1)}), \dots, (x_{\sigma(n)},y_{\sigma(n)}); \xi'},
\]
for some coupling of the randomization terms $\xi,\xi'\sim\Uniform[0,1]$. For example, this property holds for stochastic gradient descent, since the data indices are sampled uniformly at random. From this point on, all algorithms will be assumed to satisfy this measurability condition and symmetry property, without comment.

\subsection{Algorithmic stability}

In a typical statistical learning problem, an algorithm $\alg$ is applied to training data that is assumed to arise from some unknown random process. In this work, we focus on the independent and identically distributed (i.i.d.) setting, i.e.,
\[
	(X_i,Y_i) \iidsim \dgp, \quad i = 1, 2, \dots.
\]
We are interested in the stability of the predictions made by a model obtained by applying $\alg$ to a training dataset from $\dgp$. That is to say, we would like to establish that slight perturbations to the training dataset are not likely to substantially affect the resulting fitted function $\muh$. The following definition captures this notion of stability:

\begin{definition} \label{def:stability}
Let $\alg$ be a symmetric algorithm. Let $\eps \geq 0$ and $\delta \in [0,1)$. We say that $\alg$ is {\em $(\eps,\delta)$-stable} with respect to training datasets of size $n$ from a data distribution $\dgp$---or, for short, the triple $(\alg,\dgp,n)$ is $(\eps,\delta)$-stable---if
\begin{equation} \label{eq:stability}
	\PP{\abr{\muh_n(X_{n+1}) - \muh_{n-1}(X_{n+1})} > \eps} \leq \delta,
\end{equation}
where $\muh_n$ and $\muh_{n-1}$ are the fitted models obtained from the full training dataset and the dataset after removing the last data point, i.e.,
\[
	\muh_n = \algof{(X_1,Y_1), \dots, (X_n,Y_n); \xi}, \quad
	\muh_{n-1} = \algof{(X_1,Y_1), \dots, (X_{n-1},Y_{n-1}); \xi},
\]
where the data is distributed as $(X_i,Y_i) \iidsim \dgp$ and $\xi \sim \Uniform[0,1]$ is drawn independently of the data.
\end{definition}

Similarly to \citet{Elisseeff2005Stability}, we define stability by comparing the outputs of $\alg$ while fixing $\xi$ at the {\em same} value, i.e., both $\muh_n$ and $\muh_{n-1}$ are fitted using the same value $\xi$. Alternatively, we may want to define stability using a pair of {\em independent} $\xi, \xi'$, so that in~\eqref{eq:stability}, $\muh_n$ uses $\xi$ and $\muh_{n-1}$ uses $\xi'$ (e.g., the two calls to $\alg$ are initialized with two different random seeds). We consider this alternative definition, and other extensions, in Appendix~\ref{app:general_stability}.

\subsubsection{Other notions of algorithmic stability}

The earliest works on algorithmic stability date back to at least the late 1970s \citep{Rogers1978Finite, Devroye1979a, Devroye1979b}. Since then, various notions of algorithmic stability have been proposed, all describing different ways an algorithm may exhibit ``continuity" as a map from the data space $\bigcup_{n \geq 0} (\R^d \times \R)^n$ to the model space $\modelspace$. Among many such competing definitions, some of the most influential in the fields of learning theory and statistics are due to \citet{Bousquet2002Stability}, where certain forms of algorithmic stability were shown to imply good generalization behavior (see also \citet{Kearns1999Algorithmic}).

The formulation of stability that we use in this paper (Definition~\ref{def:stability}) is very similar to the {\em hypothesis stability} condition of \citet{Bousquet2002Stability}, which instead required a bound on $\tEE{|\ell(\muh_n; (X_{n+1},Y_{n+1})) - \ell(\muh_{n-1}; (X_{n+1}, Y_{n+1}))|}$ for a loss function $\ell$. Definition~\ref{def:stability} has appeared in \citet{Barber2021Predictive} under the name of {\em out-of-sample stability}, where it was shown that this condition is sufficient for predictive coverage guarantees of the jackknife (i.e., leave-one-out cross-validation) and jackknife+ methods. Earlier work by \citet{Steinberger2020Conditional} also established coverage properties for the jackknife under a similar stability property.

Many other works that also build on \citet{Bousquet2002Stability} have proposed alternative formulations of stability. \citet{Elisseeff2005Stability} extended \citet{Bousquet2002Stability}'s definitions and results to the setting of randomized algorithms, making their framework the closest to ours in the existing literature. \citet{Kutin2002a, Kutin2002b, Mukherjee2006Learning, ShalevShwartz2010Learnability} also proposed further relaxations and variants of \citet{Bousquet2002Stability}'s definitions in the context of characterizing sufficient and necessary conditions for learning problems. Interestingly, \citet[Section 6.2]{ShalevShwartz2010Learnability} give a generic but infeasible meta-algorithm in which stability plays the role of regularization, which they use to establish that a statistical problem is asymptotically {\em learnable} only if it can be learned with a stable algorithm. More recently, \citet{Deng2021Toward} advocated a functional description of algorithmic stability. It must be noted that only some of the definitions considered in these works define algorithmic stability as a property that holds jointly for an algorithm $\alg$ and a data distribution $\dgp$. For example, in \citet{Bousquet2002Stability}'s work, while {\em hypothesis stability} is defined for each algorithm-distribution pair, {\em uniform stability} is instead required to hold uniformly over all possible datasets and hence, is unsuitable as a target for statistical inference.

\subsection{Hypothesis testing framework}

The focus of this work is the following statistical inference question: given a fixed $\eps \geq 0$ and $\delta \in [0,1)$, we would like to test whether $(\alg,\dgp,n)$ is $(\eps,\delta)$-stable, with some desired bound $\alpha\in(0,1)$ on the test's error level. Writing $\dataset_\ell$ and $\dataset_u$ to denote the available labeled and unlabeled data, respectively, we will write $\widehat{\test}_{\eps,\delta} = \widehat{\test}_{\eps,\delta}(\alg,\dataset_\ell,\dataset_u) \in \{0,1\}$ for the data-dependent outcome of such a hypothesis test, where $\widehat{\test}_{\eps,\delta} = 1$ indicates that we believe stability holds. Define the ``ground truth" as
\[
	\test^*_{\eps,\delta} = \test^*_{\eps,\delta}(\alg,\dgp,n) = \One{\textnormal{$(\alg,\dgp,n)$ is $(\eps,\delta)$-stable}}.
\]
Given this target, we seek a test $\widehat{\test}_{\eps,\delta}$ that satisfies the following notion of {\em assumption-free} validity:
\begin{equation} \label{eqn:valid_hat_T}
	\textnormal{For all $(\alg,\dgp,n)$, if $\test^*_{\eps,\delta}(\alg,\dgp,n) = 0$ then $\PP{\widehat{\test}_{\eps,\delta} = 1} \leq \alpha$.}
\end{equation}
In other words, our chance of falsely labeling a $(\alg,\dgp,n)$ triple as $(\eps,\delta)$-stable using this test is bounded by $\alpha$, uniformly over all $(\alg,\dgp,n)$.

\subsection{Black-box tests}

For some simple algorithms such as nearest neighbor methods or ridge regression, their stability properties have been analyzed and are hence known \citep{Devroye1979a, Devroye1979b, Bousquet2002Stability, Hardt2016Train, Steinberger2020Conditional, Nikolakakis2022Blackbox}. For many modern learning algorithms and datasets, however, stability cannot be determined the same way, as they are too complex for this to be a tractable problem. Such algorithms are rightly regarded as ``black boxes" in the sense that our only practical means of studying them is by observing their empirical behavior, i.e., by sending different datasets through an algorithm and taking note of the changes in the output. Often these outputted fitted functions $\muh$ will also be ``black boxes" and be only accessible by evaluating them at various test points.

This ``black-box'' setting is the framework we are interested in here. In this setting, for a test $\widehat{\test}_{\eps,\delta}$ to be admissible, the computations the test makes can only rely on observable outcomes. The situation closely resembles that of black-box testing in software development where the performance of a code is tested by pushing different inputs through the code and observing the results assuming no knowledge of the inner structure \citep{Myers2011Art}. We essentially perform the same task for algorithmic stability, but with the crucial difference that we also seek assumption-free validity.

To formalize this setting, we define a {\em black-box test} of algorithmic stability. Let $\datasetspace_\ell = \cup_{m \geq 1} (\R^d \times \R)^m$ denote the space of labeled datasets of any size, and let $\datasetspace_u = \cup_{m \geq 1} (\R^d)^m$ denote the space of unlabeled datasets of any size.

\begin{definition}[Black-box test] \label{def:black_box_test}
Consider any test $\widehat{\test}$ that takes as input an algorithm $\alg$, a labeled dataset $\dataset_\ell \in \datasetspace_\ell$, and an unlabeled dataset $\dataset_u \in \datasetspace_u$, and returns a (possibly randomized) binary output $\widehat{\test}(\alg,\dataset_\ell,\dataset_u) \in \{0,1\}$. Then, we say that $\widehat{\test}$ is a {\em black-box test} if for some measurable functions $f^{(1)}, f^{(2)}, \dots$ and $g$, it can be defined in the following way:
\begin{enumerate}

\item At the initial stage $r = 1$,
\begin{enumerate}

\item Generate a labeled dataset $\dataset_\ell^{(1)} \in \datasetspace_\ell$, an unlabeled dataset $\dataset_u^{(1)} \in \datasetspace_u$, and a randomization term $\xi^{(1)}$ as a (possibly randomized) function of the input datasets:
\[
	\rbr{\dataset_\ell^{(1)}, \dataset_u^{(1)}, \xi^{(1)}} = f^{(1)}\sbr{\dataset_\ell, \dataset_u, \zeta^{(1)}},
\]
where $\zeta^{(1)} \sim \Uniform [0,1]$.

\item Fit a model and evaluate it using the generated data $(\dataset_\ell^{(1)}, \dataset_u^{(1)})$:
\[
 	\muh^{(1)} = \algof{\dataset_\ell^{(1)}; \xi^{(1)}}, \quad \Yhat^{(1)} = \muh^{(1)}\Big[\dataset_u^{(1)}\Big],
\]
where the operation of $\muh^{(1)}$ on $\dataset_u^{(1)}$ is understood in the pointwise sense.\footnote{For example, if $\dataset_u^{(1)}$ contains two unlabeled points $x_1$ and $x_2$, then $\muh^{(1)}[\dataset_u^{(1)}] = (\muh^{(1)}(x_1),\muh^{(1)}(x_2))$.}

\end{enumerate}

\item At each stage $r = 2, 3, \dots$,
\begin{enumerate}

\item Generate a labeled dataset $\dataset_\ell^{(r)} \in \datasetspace_\ell$, an unlabeled dataset $\dataset_u^{(r)} \in \datasetspace_u$, and a randomization term $\xi^{(r)}$ as a (possibly randomized) function of all the datasets and predictions observed so far:
\begin{multline*}
	\rbr{\dataset_\ell^{(r)}, \dataset_u^{(r)}, \xi^{(r)}}\\
	= f^{(r)}\sbr{\dataset_\ell, \dataset_u, \Big(\dataset_\ell^{(s)}\Big)_{s = 1}^{r-1}, \Big(\dataset_u^{(s)}\Big)_{s = 1}^{r-1}, \rbr{\Yhat^{(s)}}_{s = 1}^{r-1}, \rbr{\zeta^{(s)}}_{s = 1}^{r- 1},\rbr{\xi^{(s)}}_{s = 1}^{r-1}, \zeta^{(r)}},
\end{multline*}
where $\zeta^{(r)} \sim \Uniform [0,1]$.

\item Fit a model and evaluate it using the generated data $(\dataset_\ell^{(r)}, \dataset_u^{(r)})$:
\[
 	\muh^{(r)} = \algof{\dataset_\ell^{(r)}; \xi^{(r)}}, \quad \Yhat^{(r)} = \muh^{(r)}\Big[\dataset_u^{(r)}\Big],
\]
where the operation of $\muh^{(r)}$ on $\dataset_u^{(r)}$ is understood in the pointwise sense.

\end{enumerate}

\item Compute the output $\widehat{\test} \in \{0,1\}$ as a (possibly randomized) function of the sequence of generated datasets and observed predictions:
\[
	\widehat{\test} = g\sbr{\dataset_\ell, \dataset_u, \Big(\dataset_\ell^{(r)}\Big)_{r \geq 1}, \Big(\dataset_u^{(r)}\Big)_{r \geq 1}, \rbr{\Yhat^{(r)}}_{r \geq 1}, \rbr{\zeta^{(r)}}_{r \geq 1}, \rbr{\xi^{(r)}}_{r \geq 1}, \zeta},
\]
where $\zeta \sim \Uniform [0,1]$.

\end{enumerate}
\end{definition}
\noindent In the above, the role of the $\Uniform [0,1]$ random variables $\zeta^{(1)}, \zeta^{(2)}, \dots, \zeta \iidsim \Uniform[0,1]$ is to allow for randomization at each step of the procedure, if desired---for instance, to allow for drawing random subsamples of the data. 

\subsubsection{Examples of black-box tests}

To make the notion of a black-box test concrete, we give several examples of procedures we might use in practice to estimate an algorithm's stability. Throughout, we will write $\dataset_\ell = \big((X_1,Y_1), \dots,(X_{N_\ell},Y_{N_\ell})\big)$ to denote the available labeled data and $\dataset_u = \big(X_{N_\ell+1}, \dots, X_{N_\ell+N_u}\big)$ for the available unlabeled data.

First, we can consider splitting the available data and measuring stability on each split.

\begin{example}[The Binomial test] \label{example:Binomial_test}
Define 
\begin{equation} \label{eqn:define_kappa}
	\kappa = \kappa(n, N_\ell, N_u) = \min \cbr{\frac{N_\ell}{n}, \frac{N_\ell + N_u}{n+1}} .
\end{equation}
Then, $\lfloor\kappa\rfloor$ is the largest number of copies of independent datasets---each consisting of $n$ labeled training points and one unlabeled test point---that can be constructed from $\dataset_\ell$ and $\dataset_u$.

\begin{enumerate}

\item For $k = 1, \dots, \lfloor\kappa\rfloor$, construct the $k$-th dataset with $n$ labeled data points
\[
	(X_{(k-1) n + 1},Y_{(k-1) n + 1}), \dots, (X_{kn},Y_{kn})
\]
and one unlabeled data point $X_{\lfloor\kappa\rfloor n + k}$.

\item For each $k = 1, \dots, \lfloor\kappa\rfloor$, train models
\begin{gather*}
	\muh^{(k)}_n = \algof{(X_{(k-1) n + 1},Y_{(k-1) n + 1}), \dots, (X_{kn},Y_{kn}); \xi^{(k)}},\\
	\muh^{(k)}_{n-1} = \algof{X_{(k-1) n + 1},Y_{(k-1) n + 1}), \dots, (X_{kn - 1},Y_{kn - 1}); \xi^{(k)}}
\end{gather*}
with $\xi^{(k)} \iidsim \Uniform [0,1]$, and compute the difference in predictions \[\Delta^{(k)} = \abr{\muh^{(k)}_n\rbr{X_{\lfloor\kappa\rfloor n+k}}- \muh^{(k)}_{n-1}\rbr{X_{\lfloor\kappa\rfloor n+k}}}.\]

\item Compute how many of the $\Delta^{(k)}$'s exceed $\eps$, \[B = \sum_{k = 1}^{\lfloor\kappa\rfloor} \One{\Delta^{(k)} > \eps},\] and compare this test statistic against the $\Binomial(\lfloor\kappa\rfloor,\delta)$ distribution, returning $\widehat{\test}_{\eps,\delta}=1$ if $B$ is sufficiently small or $\widehat{\test}_{\eps,\delta}=0$ otherwise (precise details for the test are given in Section~\ref{sec:simple_test}).

\end{enumerate}
\end{example}

\noindent We will revisit this specific test in Section~\ref{sec:bounds} when we study some fundamental bounds.

Alternatively, instead of partitioning the available data, we can use bootstrapped samples to run the algorithm and generate the empirical differences $\Delta^{(k)}$.

\begin{example}[Bootstrapped sample method] \label{example:bootstrapped} \hfill

\begin{enumerate}

\item For $r = 1, \dots, R$, construct a bootstrapped training sample $(X^{(r)}_1,Y^{(r)}_1), \dots, (X^{(r)}_n,Y^{(r)}_n)$ by sampling with replacement from $\dataset_\ell$. Sample a test point $X^{(r)}_{n+1}$ from $\dataset_u$. 

\item As in Example~\ref{example:Binomial_test}, train models $\muh^{(r)}_n$ and $\muh^{(r)}_{n-1}$, and compute
 the difference in predictions
\[
	\Delta^{(r)} = \abr{\muh^{(r)}_n\rbr{X^{(r)}_{n+1}}- \muh^{(r)}_{n-1}\rbr{X^{(r)}_{n+1}}}.
\]

\item As in Example~\ref{example:Binomial_test}, compute how many of the $\Delta^{(r)}$'s exceed $\eps$,
\[
	\sum_{r = 1}^{R} \One{\Delta^{(r)} > \eps},
\] 
and return $\widehat{\test}_{\eps,\delta}=1$ if this value is sufficiently small or $\widehat{\test}_{\eps,\delta}=0$ otherwise.

\end{enumerate}
\end{example}

\noindent As another alternative, we may want to use the available data to estimate $\dgp$ first, and then test the algorithm with simulated data from the estimated distribution $\widehat{\dgp}$.

\begin{example}[Simulated sample method] \label{example:simulated} \hfill

\begin{enumerate}

\item Use $\dataset_\ell$ and $\dataset_u$ to obtain an estimate $\widehat{\dgp}$ of $\dgp$ via any method, then for $r = 1, \dots, R$, draw a simulated training set $(X^{(r)}_1,Y^{(r)}_1), \dots, (X^{(r)}_n,Y^{(r)}_n) \iidsim \widehat{\dgp}$ and a simulated test point $X^{(r)}_{n+1} \sim \widehat{\dgp}_X$, the marginal distribution of $X$ under $\widehat{\dgp}$.

\item Proceed as in Examples~\ref{example:Binomial_test} and~\ref{example:bootstrapped}.

\end{enumerate}
\end{example}

\noindent Furthermore, the black-box testing framework also allows us to use a testing procedure that learns from the earlier stages to improve the design of future queries. For instance, after running any of the above methods for some number of stages, a procedure may try to narrow down the properties of the training set and/or the test point that appear to lead to a lack of stability based on the observed outcomes up to that stage. In the subsequent stages, the dataset generation may be made to target these types of training and/or test data to learn more about the algorithm's stability properties before outputting the final answer $\widehat{\test}_{\eps,\delta} \in \{0,1\}$. Of course, these options represent only a few examples of the types of tests that we can define within the black-box framework of Definition~\ref{def:black_box_test}.

\subsubsection{Inadmissible tests}

Tests that require uncountably infinitely many executions of a black box are {\em not} admissible under Definition~\ref{def:black_box_test}. An example is a test $\widehat{\test}$ that relies on evaluating \[\sup_{y\in\R} \cbr{\algof{(X_1,Y_1), \dots, (X_{n-1},Y_{n-1}), (X_n,y); \xi}(X_{n+1})},\] since computing this supremum requires uncountably infinitely many calls to $\alg$. Similarly, evaluating \[\sup_{x\in\R^d} \left|\muh_n(x) - \muh_{n-1}(x)\right|,\] for $\muh_n = \talgof{(X_1,Y_1), \dots, (X_n,Y_n); \xi}$ and $\muh_{n-1} = \talgof{(X_1,Y_1), \dots, (X_n,Y_n); \xi}$, is also not allowed in general, as it requires uncountably infinitely many evaluations of the fitted regression functions $\muh_n$ and $\muh_{n-1}$.

\section{Limits of black-box testing} \label{sec:bounds}

In this section, we will study the hardness of testing stability in the black-box setting. To do so, we will first study the Binomial test defined in Example~\ref{example:Binomial_test}---we will see that this black-box test achieves valid inference regardless of the algorithm $\alg$ and distribution $\dgp$, but has low power (as expected, based on its inefficient use of the data). We will then prove that the performance of this simple procedure is in fact essentially optimal among all black-box tests that satisfies assumption-free validity~\eqref{eqn:valid_hat_T}, establishing the hardness of this inference problem.

\subsection{The Binomial test} \label{sec:simple_test}

Recall the Binomial test given in Example~\ref{example:Binomial_test}. This test is based on the empirical count \[B = \sum_{k = 1}^{\lfloor\kappa\rfloor} \One{\Delta^{(k)} > \eps},\] which is a $\Binomial$ random variable by construction (here $\kappa$ is defined as in~\eqref{eqn:define_kappa}). Since the empirical proportion $B / \lfloor\kappa\rfloor$ is an estimate of $\tPP{|\muh_n(X_{n+1}) - \muh_{n-1}(X_{n+1})|}$, we see that we should return $\widehat{\test}_{\eps,\delta} = 1$ if the empirical proportion $B / \lfloor\kappa\rfloor$ is sufficiently below $\delta$.

To make this precise, we first observe that since the $\Delta^{(k)}$'s are constructed on independent subsets of the data, we have $B \sim \Binomial(\lfloor\kappa\rfloor, \delta^*_\eps)$,\footnote{To include the $\lfloor\kappa\rfloor = 0$ case (i.e., $N_\ell < n$ or $N_\ell+N_u<n+1$), we abuse notation and interpret the $\Binomial(0, \delta)$ distribution as a point mass at zero.} where
\begin{equation} \label{eqn:define_delta_star}
	\delta^*_\eps = \PP{\abr{\muh_n(X_{n+1}) - \muh_{n-1}(X_{n+1})} > \eps}
\end{equation}
is the unknown true probability (that is, $\delta^*_\eps$ is the smallest possible value of $\delta'$ so that $(\alg,\dgp,n)$ is $(\eps,\delta')$-stable). Therefore, we simply need to perform a one-tailed Binomial test on $B$ to test whether $\delta^*_\eps \leq \delta$, in which case $\test^*_{\eps,\delta} = 1$, or $\delta^*_\eps > \delta$, in which case $\test^*_{\eps,\delta} = 0$. To perform this test, define $k^*_{\kappa,\delta} \in \{0, \dots, \lfloor\kappa\rfloor\}$ and $a^*_{\kappa,\delta} \in (0,1]$ as the unique values satisfying
\[
	\PP{\Binomial(\lfloor\kappa\rfloor,\delta) < k^*_{\kappa,\delta}} + a^*_{\kappa,\delta}\cdot \PP{\Binomial(\lfloor\kappa\rfloor,\delta) = k^*_{\kappa,\delta}} = \alpha,
\]
and then define
\begin{equation} \label{eqn:define_hat_T}
	\widehat{\test}_{\eps,\delta} = \begin{cases}
	1, & \text{ if } B < k^*_{\kappa,\delta},\\
	\One{\zeta \leq a^*_{\kappa,\delta}}, \text{where } \zeta \sim \Uniform [0,1], & \text{ if } B = k^*_{\kappa,\delta},\\
	0, & \text{ if } B > k^*_{\kappa,\delta}.
	\end{cases}
\end{equation}
\noindent Clearly, this test is a black-box test in the sense of Definition~\ref{def:black_box_test}. We now verify the assumption-free validity property~\eqref{eqn:valid_hat_T} to ensure bounded error for falsely declaring stability when an algorithm is {\em not} stable, and calculate this test's power to detect when an algorithm {\em is} stable.

\begin{theorem} \label{thm:simple_test}
Fix any parameters $\eps \geq 0$ and $\delta \in [0,1)$, any desired error level $\alpha \in (0,1)$, and any integers $n \geq 2$ and $N_\ell, N_u \geq 0$. Then, the black-box test $\widehat{\test}_{\eps,\delta}$ defined in~\eqref{eqn:define_hat_T} above satisfies assumption-free validity~\eqref{eqn:valid_hat_T} at level $\alpha$, that is, for any $(\alg,\dgp,n)$ that is not $(\eps,\delta)$-stable (i.e., $\test^*_{\eps,\delta} = 0$), it holds that \[\PP{\widehat{\test}_{\eps,\delta} = 1} \leq \alpha.\]
Moreover, for any $(\alg,\dgp,n)$ that is $(\eps,\delta)$-stable (i.e., $\test^*_{\eps,\delta} = 1$), if either $\delta^*_\eps = 0$ or $\delta \leq 1-\alpha^{1/\lfloor\kappa\rfloor}$ then the power of the test is given by
\begin{equation} \label{eqn:power_simple_test}
	\PP{\widehat{\test}_{\eps,\delta} = 1} = \cbr{\alpha \cdot \rbr{\frac{1-\delta^*_\eps}{1-\delta}}^{\lfloor\kappa\rfloor}} \wedge 1,
\end{equation}
where $\kappa=\kappa(n,N_\ell,N_u)$ is defined as in~\eqref{eqn:define_kappa} and $\delta^*_\eps$ is defined as in~\eqref{eqn:define_delta_star}.
\end{theorem}

\noindent
These results are straightforward consequences of the fact that $B \sim \Binomial(\lfloor\kappa\rfloor, \delta^*_\eps)$ for any $(\alg,\dgp,n)$, i.e., the distribution of $B$ depends on $(\alg,\dgp,n)$ only through the parameters $\kappa$ and $\delta^*_\eps$. We give the full proof in Appendix~\ref{app:proof_thm:simple_test}.

To understand the setting in which the power calculation applies, the two special cases are exactly the settings where the outcome $\widehat{\test}_{\eps,\delta} = 1$ can occur only when $B = 0$, i.e., when we observe $\Delta^{(k)} \leq \eps$ for all $k = 1, \dots, \lfloor\kappa\rfloor$ datasets. In particular, the condition $\delta \leq 1-\alpha^{1/\lfloor\kappa\rfloor}$ is not very restrictive in practical scenarios. For example, taking $\delta = 0.1$ (i.e., a 10\% chance of instability) and $\alpha = 0.1$ (i.e., a 10\% error rate), this condition is satisfied for any $\kappa < 21$ (i.e., the amount of data available for testing is less than 21 times than the sample size $n$ we are interested in studying).  In particular, if $\kappa$ is fairly low (e.g., $\kappa=1$), then the power is not much higher than $\alpha$---that is,
power is not much better than random.

\subsection{A bound on power for all black-box tests with assumption-free validity}

The Binomial test defined in~\eqref{eqn:define_hat_T} above appears to be a very naive and inefficient proposal---in particular, we only use a single split of the available data into $\lfloor\kappa\rfloor$ datasets of size $n+1$ (i.e., $n$ labeled training points and one unlabeled test point), and the algorithm $\alg$ is only called $2 \lfloor\kappa\rfloor$ many times. We might expect to increase the power of the test by re-splitting or by resampling as in Example~\ref{example:bootstrapped}, or perhaps by creating new data through simulation or perturbation as in Example~\ref{example:simulated}, so that it may be possible to observe the algorithm's stability behavior on a wider range of different datasets.

However, our next result establishes that, despite its simplicity, the Binomial test proposed in~\eqref{eqn:define_hat_T} can in fact be optimal, in the sense that no universally valid black-box test can improve its power to detect that an algorithm is stable.

\begin{theorem} \label{thm:optimality}
Fix any parameters $\eps \geq 0$ and $\delta \in [0,1)$, any desired error level $\alpha \in (0,1)$, and any integers $n \geq 2$ and $N_\ell, N_u \geq 0$. Let $\widehat{\test}_{\eps,\delta}$ be any black-box test as in Definition~\ref{def:black_box_test} satisfying assumption-free validity~\eqref{eqn:valid_hat_T} at level $\alpha$. Then, for any $(\alg,\dgp,n)$ that is $(\eps,\delta)$-stable (i.e., $\test^*_{\eps,\delta} = 1$), the power of $\widehat{\test}_{\eps,\delta}$ is bounded as
\begin{equation} \label{eqn:power_optimal}
	\PP{\widehat{\test}_{\eps,\delta} = 1} \leq \cbr{\alpha \cdot \rbr{\frac{1-\delta^*_\eps}{1-\delta}}^{\kappa}} \wedge 1,
\end{equation}
where $\kappa=\kappa(n,N_\ell,N_u)$ is defined as in~\eqref{eqn:define_kappa} and $\delta^*_\eps$ is defined as in~\eqref{eqn:define_delta_star}.
\end{theorem}

We prove this theorem in Appendix~\ref{app:proof_thm:optimality}. The key idea is that for any triple $(\alg,\dgp,n)$ that is $(\eps,\delta)$-stable, we can perturb the algorithm $\alg$ and the data distribution $\dgp$ very slightly so that the perturbed triple $(\alg',\dgp',n)$ is no longer $(\eps,\delta)$-stable. However, with only $N_\ell$ labeled data points and $N_u$ unlabeled data points available, it is difficult to tell apart the outcomes of running the original algorithm $\alg$ on data from $\dgp$ from the outcomes of running the perturbed algorithm $\alg'$ on data from $\dgp'$. Since validity of the test ensures that $\tPP{\widehat{\test}_{\eps,\delta} = 1}\leq\alpha$ for the perturbed triple $(\alg',\dgp',n)$, it therefore follows that this probability cannot be much larger than $\alpha$ for the original triple $(\alg,\dgp,n)$.

\subsubsection{Optimality of the Binomial test}

To understand the implications of Theorem~\ref{thm:optimality}, let us compare the upper bound on power~\eqref{eqn:power_optimal}, which holds universally for any black-box test of stability with assumption-free validity, to the power \eqref{eqn:power_simple_test} achieved by our very simple Binomial test. In particular, in the case where $\kappa$ is an integer, we see that the Binomial test defined in Section~\ref{sec:simple_test} is in fact optimal over all black-box tests in the regime $\delta \leq 1-\alpha^{1/\kappa}$ or $\delta^*_\eps = 0$. That is to say, there is no possibility of power gain by running a more complicated procedure in this scenario.

For example, if $N_\ell = n$ (i.e., the number of available labeled data points matches the sample size at which we want to test stability) and $N_u\geq 1$, then $\kappa = 1$, and we can test the stability of $\alg$ with maximum power using only two calls to $\alg$ (one on the $n$ points to fit $\muh_n$ and the other on the first $n-1$ points to fit $\muh_{n-1}$) in the regime $\delta \leq 1-\alpha$. This is a surprising result---we would intuitively expect that additional calls to $\alg$, on more carefully curated datasets obtained via some more sophisticated resampling procedure, might yield useful information about stability beyond simply calling $\alg$ on the single available dataset. However, by comparing the power achieved in Theorem~\ref{thm:simple_test} and the universal upper bound shown in Theorem~\ref{thm:optimality}, we see that any additional calls to $\alg$ are not adding any information to the test.

\subsubsection{The parameter $\kappa$}

Is Theorem~\ref{thm:optimality} a pessimistic result? The answer depends on {\em how} one wishes to use a stability test, which in turn 
will affect the value of the parameter $\kappa$, which controls the upper bound on power. Broadly speaking, we can imagine two different types of goals.

First, we might simply be interested in characterizing the algorithm, by studying its behavior and its properties on different data distributions. In this setting, it is plausible that we might be interested in learning about the algorithm's stability at a sample size $n$ that is far smaller than the available dataset size $N_\ell+N_u$. Here, the power of a black-box test can be quite high as long as the sample size ratio $\kappa$ is large.

Alternatively, it may be the case that we are not interested in learning about the algorithm's stability for its own sake, but rather that we need to verify that stability holds in order to check assumptions for other procedures. For example, for the task of predicting plausible values of $Y$ given a new value of $X$, methods based on cross-validation can be used to construct predictive intervals around $\muh_n(X)$, but their coverage guarantees typically require algorithmic stability \citep{Steinberger2016LeaveOneOut, Steinberger2020Conditional, Barber2021Predictive}. In this case, we would probably want to test the algorithm's stability at the same sample size as the number of available labeled data points (i.e., $n = N_\ell$), since we would expect that using the largest possible sample size for training would yield the most accurate fitted model. This is the scenario for which the implications of Theorem~\ref{thm:optimality} are quite discouraging. When $\kappa = 1$, the maximum possible power is upper-bounded by $\alpha / (1-\delta)$, which, if $\delta$ is small, is not much higher than the error level $\alpha$ allowed by the validity condition~\eqref{eqn:valid_hat_T}. In other words, it is impossible for our test to be substantially better than random if we need to check the stability assumption at a sample size $n \approx N_\ell$.

\section{Extensions}

In this section, we consider several extensions and alternative frameworks for testing algorithmic stability.

\subsection{Alternative targets for inference}

So far, we have focused on finding an answer to the binary question ``Does $(\alg,\dgp,n)$ satisfy $(\eps,\delta)$-stability?" with a fixed $\eps$ and $\delta$. However, this is only one of many possible ways to formulate this inference problem, and here we list two additional questions that are also of interest.

\begin{quote}
Given a fixed $\eps \geq 0$, is it possible to compute a data-dependent estimate $\widehat{\delta}_\eps$ such that $(\eps,\widehat{\delta}_\eps)$-stability holds for $(\alg,\dgp,n)$ with some desired confidence level $1-\alpha$?
\end{quote}

\begin{quote}
Given a fixed $\delta \in [0,1)$, is it possible to compute a data-dependent estimate $\widehat{\eps}_\delta$ such that $(\widehat{\eps}_\delta,\delta)$-stability holds for $(\alg,\dgp,n)$ with some desired confidence level $1-\alpha$?
\end{quote}

\noindent Similarly to how we defined the ``ground truth" $\test^*_{\eps,\delta}$ for the data-dependent outcome $\widehat{\test}_{\eps,\delta}$, we define the associated targets as
\[
	\delta^*_\eps = \inf\cbr{\delta \geq 0: \test^*_{\eps,\delta} = 1}, \quad \eps^*_\delta = \inf\cbr{\eps \geq 0: \test^*_{\eps,\delta} = 1}.
\]
Examining the definition of $\test^*_{\eps,\delta}$, we can observe that $\delta^*_\eps = \tPP{|\muh_n(X_{n+1}) - \muh_{n-1}(X_{n+1})| > \eps}$ (as considered earlier in~\eqref{eqn:define_delta_star}), while $\eps^*_\delta$ is the $(1-\delta)$-quantile of $|\muh_n(X_{n+1}) - \muh_{n-1}(X_{n+1})|$. By definition of $(\eps,\delta)$-stability, we see that for any $(\eps,\delta)$, it holds that
\begin{equation} \label{eqn:ground_truth_equiv}
	\test^*_{\eps,\delta} = 1 \quad \iff \quad \delta \geq \delta^*_\eps \quad \iff \quad \eps \geq \eps^*_\delta.
\end{equation}

Next, we define notions of assumption-free validity for inference on each of the new targets, which are the analogues of the validity condition~\eqref{eqn:valid_hat_T} for the test $\widehat{\test}_{\eps,\delta}$ studied earlier.
For the problem of estimating $\delta^*_\eps$ given a fixed $\eps \geq 0$, we seek an estimator $\widehat{\delta}_\eps$ such that
\begin{equation} \label{eqn:valid_hat_delta}
	\textnormal{$\PP{\widehat{\delta}_\eps \geq \delta^*_\eps} \geq 1-\alpha$ for all $(\alg,\dgp,n)$.}
\end{equation}
(Equivalently, $\tPP{\test_{\eps, \widehat{\delta}_\eps}^* = 1} \geq 1-\alpha$.) In other words, given an algorithm $\alg$ and data drawn from $\dgp$, we conclude with $1-\alpha$ level confidence that $(\alg,\dgp,n)$ is $(\eps, \widehat{\delta}_\eps)$-stable.
Similarly, for the problem of estimating $\eps^*_\delta$ given a fixed $\delta \in [0,1)$, we seek an estimator $\widehat{\eps}_\delta$ such that
\begin{equation} \label{eqn:valid_hat_eps}
	\textnormal{$\PP{\widehat{\eps}_\delta \geq \eps^*_\delta} \geq 1-\alpha$ for all $(\alg,\dgp,n)$.}
\end{equation}
(Equivalently, $\tPP{\test_{\widehat{\eps}_\delta, \delta}^* = 1} \geq 1-\alpha$.) In other words, given an algorithm $\alg$ and data drawn from $\dgp$, we conclude with $1-\alpha$ level confidence that $(\alg,\dgp,n)$ is $(\widehat\eps_\delta,\delta)$-stable.

It turns out that these three inference targets $\test^*_{\eps,\delta}$, $\delta^*_\eps$, and $\eps^*_\delta$ are highly interconnected---any procedure that is able to achieve one can be adapted to achieve the other two. The following proposition says that any assumption-free estimation procedure (for either parameter $\delta^*_\eps$ or $\eps^*_\delta$) can be turned into an assumption-free test $\widehat{\test}_{\eps,\delta}$, and vice versa.

\begin{proposition} \label{prop:T_delta_eps}
The following equivalence holds for inference procedures $\widehat{\test}_{\eps,\delta}$, $\widehat{\delta}_\eps$, and $\widehat{\eps}_\delta$:

\begin{enumerate}[label=(\alph*)]
\item For any fixed $\eps \geq 0$, suppose there is an estimator $\widehat{\delta}_\eps$ satisfying~\eqref{eqn:valid_hat_delta}. Then, for any $\delta \in [0,1)$, the test $\widehat{\test}_{\eps,\delta} = \tOne{\widehat{\delta}_\eps \leq \delta}$ satisfies~\eqref{eqn:valid_hat_T}.

\item For any fixed $\delta \in [0,1)$, suppose there is an estimator $\widehat{\eps}_\delta$ satisfying~\eqref{eqn:valid_hat_eps}. Then, for any $\eps \geq 0$, the test $\widehat{\test}_{\eps,\delta} = \tOne{\widehat{\eps}_\delta \leq \eps}$ satisfies~\eqref{eqn:valid_hat_T}.

\item Suppose $\{\widehat{\test}_{\eps,\delta}: \eps \geq 0,\ \delta \in [0,1)\}$ is a family of tests such that each $\widehat{\test}_{\eps,\delta}$ satisfies~\eqref{eqn:valid_hat_T}. Then, for any fixed $\eps \geq 0$, the estimator $\widehat{\delta}_\eps = \inf\{\delta: \widehat{\test}_{\eps,\delta'} = 1\ \forall \ \delta'\geq \delta\}$ (or $\widehat{\delta}_\eps = 1$, if this set is empty) satisfies~\eqref{eqn:valid_hat_delta}. 

\item Under the same assumptions as in (c), for any fixed $\delta \in [0,1)$, the estimator $\widehat{\eps}_\delta = \inf\{\eps : \widehat{\test}_{\eps',\delta} = 1\ \forall \ \eps'\geq \eps\}$ (or $\widehat{\eps}_\delta = \infty$, if this set is empty) satisfies~\eqref{eqn:valid_hat_eps}.

\end{enumerate}
\end{proposition}

\noindent The proof is in Appendix~\ref{app:proof_prop:T_delta_eps}.

Proposition~\ref{prop:T_delta_eps} implies that the problems of assumption-free inference for any of these three targets are essentially equivalent. Therefore, although we have exclusively focused on the question of testing stability at a fixed $(\eps,\delta)$ in the earlier sections, our conclusions apply to the other two inference problems as well to bound the power of any black-box procedure. Specifically, assuming that $\widehat{\delta}_{\eps}$ (respectively, $\widehat{\eps}_{\delta}$) is a black-box procedure in the same sense as Definition~\ref{def:black_box_test}, combining Theorem~\ref{thm:optimality} with part (a) (respectively, part (b)) of Proposition~\ref{prop:T_delta_eps} will yield an upper bound on the probability of the event $\widehat{\delta}_{\eps} \leq \delta$ (respectively, $\widehat{\eps}_{\delta} \leq \eps$), thus establishing bounds on the power
of any such procedure. Conversely, combining part (c) (respectively, part (d)) of Proposition~\ref{prop:T_delta_eps} with the Binomial test $\widehat{\test}_{\eps,\delta}$ constructed in Section~\ref{sec:simple_test}, offers a simple construction for an estimator $\widehat{\delta}_\eps$ (respectively, $\widehat{\eps}_\delta$) satisfying the validity condition~\eqref{eqn:valid_hat_delta} (respectively,~\eqref{eqn:valid_hat_eps}).

\subsection{The role of uncountability} \label{sec:discussion_uncountable}

Our optimality result (Theorem~\ref{thm:optimality}) is stated in the setting of data $(X,Y)\in\R^d\times\R$.
While it is not necessary for the feature and response to lie in $\R^d$ and in $\R$ specificially, 
our proof does rely on the {\em uncountability} of these spaces. Specifically, 
the proof of the  theorem relies on finding a point $x_*$ in the feature space $\R^d$, and/or a point $y_*$ in the response space $\R$, which has zero probability of being observed at any point throughout the course of the black-box testing procedure (i.e., at any stage of Definition~\ref{def:black_box_test}). We then define the perturbed distribution $P'$ to place a small probability at $x_*$ and/or $y_*$, and define the perturbed algorithm $\alg'$ to output a corrupted model if the input data includes $x_*$ and/or $y_*$.
In particular, establishing the existence of such a point relies on the uncountability of the feature space $\R^d$ (for finding $x_*$) or of the response space $\R$ (for finding $y_*$). 

Of course, in practice, many applied problems have either the features or the response (or both) lying in a countable space.
 If the support of $X$ or $Y$ is extremely large (e.g., a real-valued quantity rounded to floating point precision), then we would
expect the same type of bound on power to still hold (see Section~\ref{sec:discussion_uncountable_part2} for further discussion). 
In other settings, however,
we might have a discrete random variable with a small support---for instance, an image classification
task with a real-valued image $X \in [0,1]^d$ as the feature and a binary label $Y \in \{0,1\}$ as the response. 
In this setting, the proof technique of finding $x_*$ or $y_*$ can no longer be carried out.

However, we can split into cases to obtain partial results. Specifically, suppose we observe data points $(X,Y) \in \X \times \Y$ for some feature space $\X \subseteq \R^d$ and some response space $\Y \subseteq \R$. The black-box test $\widehat{\test}_{\eps,\delta} $ is then only required to satisfy validity~\eqref{eqn:valid_hat_T} with respect to algorithms that train on data lying in $\X \times \Y$. If $\Y$ is assumed to be uncountable, then we can show (via the relevant part of the proof of Theorem~\ref{thm:optimality}) that power is bounded as
\begin{equation} \label{eqn:power_optimal_Ycal}
	\PP{\widehat{\test}_{\eps,\delta} = 1} \leq \cbr{\alpha \cdot \rbr{\frac{1-\delta^*_\eps}{1-\delta}}^{N_\ell/n}} \wedge 1.
\end{equation}
If instead $\X$ is assumed to be uncountable, then we can show (via the relevant part of the proof of Theorem~\ref{thm:optimality}) that power is bounded as
\begin{equation} \label{eqn:power_optimal_Xcal}
	\PP{\widehat{\test}_{\eps,\delta} = 1} \leq \cbr{\alpha \cdot \rbr{\frac{1-\delta^*_\eps}{1-\delta}}^{(N_\ell+N_u)/(n+1)}} \wedge 1.
\end{equation}
Of course, if $\X$ and $\Y$ are both uncountable, then combining these two upper bounds yields the previous result~\eqref{eqn:power_optimal}.

\subsection{Tests for restricted model classes}

Our optimality result (Theorem~\ref{thm:optimality}) also exploits our agnosticism about $\alg$, such as our unwillingness to make assumptions about the model class $\modelclass \subseteq \modelspace$ that $\alg$ maps onto. However, it is not uncommon to have additional
 knowledge about the kind of models $\alg$ produces, and in some settings
 these models might even be known to lie in a simple class (e.g., linear functions). This opens up the possibility of constructing
a test of stability that makes use of this extra information: can we gain power, and avoid the upper bound of Theorem~\ref{thm:optimality},
by leveraging information about the types of functions $\muh$ that can be returned by $\alg$?
Since model class complexity has long been central to our understanding of the difficulty of various
statistical learning problems, and we may expect it to play some role here as well.

Surprisingly, we find that restricting our attention to only those algorithms producing models in some specific class does little to improve our power, even for extremely simple classes such as the class of all constant models or of all linear models. This is because the difficulty of our black-box testing framework stems mostly from the
fact that {\em the algorithm $\alg$ itself} is a black box that may be arbitrarily complex,
while the complexity of the fitted function $\muh$ plays only a minor role in the hardness of the testing problem.

\subsubsection{Black-box tests for transparent models}

Before we can re-analyze the power, we introduce a more potent variant of the black-box testing strategy in which the final decision $\widehat{\test} \in \{0, 1\}$ is computed as a {\em direct} function of observed fitted models $\muh^{(r)}$'s (although these models themselves are still obtained by  evaluating $\alg$ on different labeled datasets $\dataset_\ell^{(r)}$'s). We can think of such a test as having the ability to ``see through" any $\muh$ without being able to do the same for $\alg$, and thus we refer to this as the ``transparent model'' setting.

\begin{definition}[Black-box tests for transparent models] \label{def:black_box_test_transparent}
We say that $\widehat{\test}$ is a {\em black-box test for transparent models} if for some functions $f^{(1)}, f^{(2)}, \dots$ and $g$, it can be defined in the following way:
\begin{enumerate}
\item At the initial stage $r = 1$,

\begin{enumerate}
\item Generate a labeled dataset $\dataset_\ell^{(1)} \in \datasetspace_\ell$ and a randomization term $\xi^{(1)}$ as a (possibly randomized) function of the input datasets:
\[
	\rbr{\dataset_\ell^{(1)}, \xi^{(1)}} = f^{(1)}\sbr{\dataset_\ell, \dataset_u, \zeta^{(1)}},
\]
where $\zeta^{(1)} \sim \Uniform [0,1]$.

\item Fit a model using the generated labeled data $\dataset_\ell^{(1)}$: $\muh^{(1)} = \algof{\dataset_\ell^{(1)}; \xi^{(1)}}$.
\end{enumerate}

\item At each stage $r = 2, 3, \dots$,

\begin{enumerate}
\item Generate a labeled dataset $\dataset_\ell^{(r)} \in \datasetspace_\ell$ and a randomization term $\xi^{(r)}$ as a (possibly randomized) function of all the datasets and fitted models observed so far:
\begin{multline*}
	\rbr{\dataset_\ell^{(r)}, \xi^{(r)}}
	= f^{(r)}\sbr{\dataset_\ell, \dataset_u, \Big(\dataset_\ell^{(s)}\Big)_{s = 1}^{r-1}, \rbr{\muh^{(s)}}_{s = 1}^{r-1}, \rbr{\zeta^{(s)}}_{s = 1}^{r- 1},\rbr{\xi^{(s)}}_{s = 1}^{r-1}, \zeta^{(r)}},
\end{multline*}
where $\zeta^{(r)} \sim \Uniform [0,1]$.

\item Fit a model using the generated labeled data $\dataset_\ell^{(r)}$:
$\muh^{(r)} = \algof{\dataset_\ell^{(r)}; \xi^{(r)}}$.

\end{enumerate}

\item Compute the output $\widehat{\test} \in \{0,1\}$ as a (possibly randomized) function of the sequence of generated datasets and fitted models:
\[
	\widehat{\test} = g\sbr{\dataset_\ell, \dataset_u, \Big(\dataset_\ell^{(r)}\Big)_{r \geq 1}, \rbr{\muh^{(r)}}_{r \geq 1}, \rbr{\zeta^{(r)}}_{r \geq 1}, \rbr{\xi^{(r)}}_{r \geq 1}, \zeta},
\]
where $\zeta \sim \Uniform [0,1]$.\footnote{We cannot assume that $f^{(1)}, f^{(2)}, \dots, g$ are measurable because we are not treating $\modelspace$ as a measure space. However, we instead assume that the functions $f^{(1)}, f^{(2)}, \dots, g$ must be chosen such that the resulting map $(\dataset_\ell, \dataset_u, \zeta_1, \zeta_2, \dots, \zeta) \mapsto \widehat{\test} \in \{0, 1\}$ is measurable.} 

\end{enumerate}
\end{definition}

\noindent Compare this to Definition~\ref{def:black_box_test}: there, the fitted models $\muh^{(r)}$
could only be observed indirectly, through {\em evaluations} of each $\muh^{(r)}$ on different generated unlabeled data $\dataset_u^{(r)}$,
while in this new setting, $\muh^{(r)}$ can be observed directly.
On the other hand, both here and in Definition~\ref{def:black_box_test},
the algorithm $\alg$ can only be studied through evaluations on different generated labeled datasets $\dataset_\ell^{(r)}$ to obtain models $\muh^{(r)}$.

Even with this additional capability to examine fitted models $\muh$ analytically,
our next result shows that tests required to satisfy the assumption-free validity condition \eqref{eqn:valid_hat_T} for {\em all} algorithms
show only a miniscule improvement in achievable power.

\begin{theorem} \label{thm:optimality_transparent}
Fix any parameters $\eps \geq 0$ and $\delta \in [0,1)$, any desired error level $\alpha \in (0,1)$, and any integers $n \geq 2$ and $N_\ell, N_u \geq 0$. Let $\widehat{\test}_{\eps,\delta}$ be any  any black-box test for transparent models as in Definition~\ref{def:black_box_test_transparent}, satisfying assumption-free validity~\eqref{eqn:valid_hat_T} at level $\alpha$---that is, $\widehat{\test}_{\eps,\delta}$ is valid with respect to {\em all} algorithms.

Then, for any $(\alg,\dgp,n)$ that is $(\eps,\delta)$-stable (i.e., $\test^*_{\eps,\delta} = 1$), the power of $\widehat{\test}_{\eps,\delta}$ is bounded as
\begin{equation} \label{eqn:power_optimal_transparent}
	\PP{\widehat{\test}_{\eps,\delta} = 1} \leq \cbr{\alpha \cdot \rbr{\frac{1-\delta^*_\eps}{1-\delta}}^{\kappa_\ell}} \wedge 1,
\end{equation}
where $\kappa_\ell = \kappa_\ell (n,N_\ell) = {N_\ell} / {n}$, and $\delta^*_\eps$ is defined as in~\eqref{eqn:define_delta_star}.
\end{theorem}

\noindent  This result is proved in Appendix~\ref{app:proofs_transparent}. Note that this upper bound is identical to the bound \eqref{eqn:power_optimal_Ycal} obtained in the setting where the feature vector $X$ lies in a countable space (but $Y$ is still assumed to lie in an uncountable space). Indeed, the connection is straightforward: if $X$ lies in a countable space, then according to our definition of a black-box test (Definition~\ref{def:black_box_test}, which allows countably infinitely many evaluations of each $\muh$), it is in principle possible to evaluate $\muh(x)$ for every possible feature vector $x$---and so we are essentially in the transparent-model setting whenever $X$ lies in a countable space.

In practice, labeled data tend to be more scarce than unlabeled data, so we would expect ${N_\ell} / {n} < \trbr{N_\ell+N_u} / \trbr{n+1}$. If this is the case, then $\kappa_\ell = \kappa$, and there is no benefit from using a black-box test for transparent models in place of a complete black-box test. Even in the rare case that $\trbr{N_\ell+N_u} / \trbr{n+1}$ is the smaller ratio, because $\kappa_\ell \leq \kappa \cdot \trbr{n+1} / n$, the improved upper bound \eqref{eqn:power_optimal_transparent} is only incrementally higher than the previous upper bound \eqref{eqn:power_optimal}.

\subsubsection{Black-box tests for constrained models}

Theorem~\ref{thm:optimality_transparent} says that even when a black-box test has the capability to ``see through" any fitted model, this capability translates to only a mild improvement in terms of power, if the test is required to be valid for {\em all} algorithms $\alg$ (as well as for all probability distributions $\dgp$ and sample sizes $n$).

However, in settings where $\alg$ is known to always return fitted models $\muh$ from a particular class,
it does not seem necessary to require our test to be valid over {\em all} algorithms $\alg$. Instead,
we will now ask about the fundamental limits of black-box testing when we restrict to only those algorithms that produce models in a particular class. For example, if 
it is known that our black-box algorithm $\alg$ always returns linear models, then we can use a test that is guaranteed to be valid for all algorithms $\alg$
returning linear models. This is arguably a more sensible goal, and because we are requiring less of our test, it seems reasonable to hope for an improvement in power.

Given a model class $\modelclass \subseteq \modelspace$, let $\algspace_{\modelclass}$ be the set of algorithms that map into $\modelclass$, i.e., if $\alg \in \algspace_{\modelclass}$, then $\muh = \alg[\dataset; \xi] \in \modelclass$ for any labeled dataset $\dataset$ and any randomization term $\xi$. Then, we say that a test $\widehat{\test}_{\eps,\delta}$ is valid at level $\alpha$ {\em restricted to the model class $\modelclass$} if
\begin{equation} \label{eqn:valid_hat_T_modelclass}
	\textnormal{For all $(\alg,\dgp,n)$ with $\alg \in \algspace_{\modelclass}$, if $\test^*_{\eps,\delta}(\alg,\dgp,n) = 0$ then $\PP{\widehat{\test}_{\eps,\delta} = 1} \leq \alpha$}.
\end{equation}

To prove a result analogous to Theorems~\ref{thm:optimality} and \ref{thm:optimality_transparent}, we need one condition on the model class $\modelclass$:
\begin{multline} \label{eqn:modelclass_assume}
\textnormal{For any $m \geq 1$ and $t, L > 0$, there is a symmetric function $M: (\modelclass)^m \mapsto \modelclass$}\\
\textnormal{such that, for all $\mu_1, \dots, \mu_m \in \modelclass$, the function $\mu = M(\mu_1, \dots, \mu_m)$}\\
\textnormal{satisfies $\mu(x) > t + \max_{i = 1, \dots, m}\{\mu_i(x)\}$ for all $\|x\|_2\leq L$.\footnote{As a regularity condition, we also assume that the map $((\dataset_1,\xi_1), \dots, (\dataset_m,\xi_m), x) \mapsto M\trbr{\talgof{\dataset_1; \xi_1}, \dots, \talgof{\dataset_m; \xi_m}}(x)$ is a measurable function for any algorithm $\alg$ satisfying the measurability conditions of Section~\ref{sec:background}.}}
\end{multline}
This is a mild condition. In particular, it is satisfied by any continuous class that includes all constant functions because we can take $M(\mu_1, \dots, \mu_m)$ to be the constant function $\mu(x) = 1 + t + \max_{i = 1, \dots, m} \{\sup_{x': \|x'\|_2 \leq L} \mu_i(x')\}$. For example, the class of all linear functions satisfies \eqref{eqn:modelclass_assume}. 

\begin{theorem} \label{thm:optimality_constrained}
Fix any parameters $\eps \geq 0$ and $\delta \in [0,1)$, any desired error level $\alpha \in (0,1)$, and any integers $n \geq 2$ and $N_\ell, N_u \geq 0$. Let $\widehat{\test}_{\eps,\delta}$ be any  black-box test for transparent models as in Definition~\ref{def:black_box_test_transparent},
satisfying assumption-free validity~\eqref{eqn:valid_hat_T_modelclass} at level $\alpha$
with respect to a constrained model class $\modelclass$ satisfying \eqref{eqn:modelclass_assume}.

Then, for any $(\alg,\dgp,n)$ with $\alg \in \algspace_{\modelclass}$ that is $(\eps,\delta)$-stable (i.e., $\test^*_{\eps,\delta} = 1$), the power of $\widehat{\test}_{\eps,\delta}$ is bounded as
\begin{equation} \label{eqn:power_optimal_transparent}
	\PP{\widehat{\test}_{\eps,\delta} = 1} \leq \cbr{\alpha \cdot \rbr{\frac{1-\delta^*_\eps}{1-\delta}}^{\kappa_\ell}} \wedge 1,
\end{equation}
where $\kappa_\ell = \kappa_\ell (n,N_\ell) = {N_\ell} / {n}$, and $\delta^*_\eps$ is defined as in~\eqref{eqn:define_delta_star}.
\end{theorem}

\noindent This result is proved in Appendix~\ref{app:proofs_transparent}. Note that this is the same bound as in Theorem~\ref{thm:optimality_transparent}---in other 
words, knowing the class of models $\modelclass$ that $\alg$ maps into, and 
requiring our test to only be valid with respect to this class, leads to {\em no improvement} in the upper bound on power.
Ultimately, the challenge of the black-box testing framework lies in the fact that it is the model fitting process $\alg$ itself which is a black box. 
Comparing our various results, we conclude
 that
treating the fitted functions $\muh$ as black boxes (as in Theorem~\ref{thm:optimality}), or as transparent or constrained functions
(as  in Theorems~\ref{thm:optimality_transparent} and~\ref{thm:optimality_constrained}), has little effect on the difficulty of the problem.

\section{Discussion}\label{sec:discussion}

In this work, we have shown a universal bound on the power achievable by any black-box test of stability (Theorem~\ref{thm:optimality}), and have seen that, surprisingly, in some regimes the maximum power is achieved by a simple Binomial test that only calls the algorithm twice for each batch of $n$ labeled data points and one unlabeled test point (Theorem~\ref{thm:simple_test}). This shows that, in a setting with limited available data, it is essentially impossible to test stability in the black-box setting with power that's substantially better than random. To conclude, we will now discuss some open questions and further directions for exploration.

\subsection{The role of uncountability, revisited} \label{sec:discussion_uncountable_part2}

As discussed in Section~\ref{sec:discussion_uncountable} earlier, our upper bound in Theorem~\ref{thm:optimality} relies on uncountability of the feature space $\R^d$ and the response space $\R$. If instead the data lies in some space $\X \times \Y$, where only $\Y$ or only $\X$ is uncountable, then the partial results~\eqref{eqn:power_optimal_Ycal} and~\eqref{eqn:power_optimal_Xcal} can instead be obtained. On the other hand, if $\X$ and $\Y$ are {\em both} countable, then none of our upper bounds apply, since we can exhaustively test all possible datasets with countably many calls to $\alg$.

Note that this setting is different than simply assuming that $X$ or $Y$ is discrete but not assuming a known support. For example, if we assume $(X,Y) \in \R^d \times \R$ is discrete but has an unknown support, then even if we require the validity statement~\eqref{eqn:valid_hat_T} to hold only for triples $(\alg,\dgp,n)$ where $\dgp$ is a discrete distribution on $\R^d \times \R$, the same upper bound on power given in Theorem~\ref{thm:optimality} will still hold.

However, it is worth noting that our definition of a black-box test allowed for countably infinitely many calls to $\alg$ and/or evaluations of fitted models $\muh$, which is of course not feasible in practice. An interesting question remains: if $\X$ and $\Y$ are countable but have large (or infinite) supports, then we would expect that a {\em finite} computational budget for our black-box test would lead to similar bounds on power. In particular, in any practical setting, real-valued data are stored in a discretized way, e.g., with floating-point precision. This means that all spaces are, technically, countable, but the support size is so vast that we expect the problem to behave essentially like the uncountable setting because computational constraints prohibit us from trying every possible value (at floating-point precision) of the data.
Our proof technique for establishing the power bound of Theorem~\ref{thm:optimality} could then be extended
to this setting, to quantify the tradeoff between the support size of $\X$ and/or $\Y$ versus the computational budget bounding how many times $\alg$ can 
be called. We leave this important question for future work.

\subsection{Data-conditional stability}

We have just seen that it is essentially impossible to test whether $\alg$ is stable at sample size $n$, if $N_\ell=n$ is the number of labeled data points available for the test. However, if the reason that we want to test stability is for verifying that predictive inference methods such as jackknife can be applied to $\alg$, testing whether $(\alg,\dgp,n)$ is $(\eps,\delta)$-stable may be too strong of a goal. Specifically, given a training dataset $(X_1,Y_1), \dots, (X_n,Y_n)$ and a test point $(X_{n+1},Y_{n+1})$ (where $Y_{n+1}$ is of course unobserved, since it is the target of predictive inference), examining the proof of \citet[Theorem 5]{Barber2021Predictive} we see that the predictive coverage of the jackknife method relies on bounding
\begin{equation} \label{eqn:stability_for_jackknife}
	\abr{\muh_{n;-i}(X_i) - \muh_{n+1;-i}(X_i)}
\end{equation}
where $\muh_{n;-i}$ denotes the model fitted by running $\alg$ on data points $\big((X_j,Y_j): 1\leq j \leq n, \ j \neq i\big)$ and $\muh_{n+1;-i}$ denotes the model fitted by running $\alg$ on data points $\big((X_j,Y_j): 1 \leq j \leq n+1, \ j \neq i\big)$.

Requiring stability of $(\alg,\dgp,n)$, as in Definition~\ref{def:stability}, ensures that the quantity~\eqref{eqn:stability_for_jackknife} is small with high probability {\em over a random draw of the entire dataset}. On the other hand, when using stability in the proof of the jackknife's predictive coverage, we only need to know that the quantity~\eqref{eqn:stability_for_jackknife} is small {\em on the actual dataset}, which we are able to observe with the exception of the last point $n+1$. In other words, we only need to test whether $\alg$ is stable with respect to a random draw of this single last data point (and we can condition on the observed data). This notion of ``data-conditional stability'' appears to be a weaker property than the (unconditional) stability property of Definition~\ref{def:stability}, and therefore it may be possible to test this conditional property with higher power than the bounds established in Theorem~\ref{thm:optimality}. We leave this question for future work.

\subsection*{Acknowledgements}

B.K. acknowledges support from the National Institutes of Health via grants R01GM114029 and R01GM133848.
R.F.B. was supported by the National Science Foundation via grants DMS-1654076 and DMS-2023109, and by the Office of Naval Research via grant N00014-20-1-2337.

\bibliography{stability}
\bibliographystyle{plainnat}

\appendix
\section{Proof of Theorem~\ref{thm:simple_test}} \label{app:proof_thm:simple_test}

As calculated in Section~\ref{sec:simple_test}, we have shown that, for any algorithm $\alg$ and distribution $\dgp$, it holds that $B \sim \Binomial(\lfloor\kappa\rfloor, \delta^*_\eps)$ (note that this distribution depends on $(\alg,\dgp,n)$ implicitly via the parameters $\kappa$ and $\delta^*_\eps$).

Next,we check validity, meaning that we need to show that whenever $\test^*_{\eps,\delta} = 0$, we have $\tPP{\widehat{\test}_{\eps,\delta} = 1} \leq \alpha$. By definition, $\test^*_{\eps,\delta} = 0$ implies that $\delta^*_\eps > \delta$, and so we have
\begin{align*}
	\PP{\widehat{\test}_{\eps,\delta} = 1}
	&= \PP{B < k^*_{\kappa,\delta}} + a^*_{\kappa,\delta} \cdot \PP{B = k^*_{\kappa,\delta}}\\
	&= (1-a^*_{\kappa,\delta}) \cdot \PP{B < k^*_{\kappa,\delta}} + a^*_{\kappa,\delta} \cdot \PP{B \leq k^*_{\kappa,\delta}}\\
	&\leq (1-a^*_{\kappa,\delta}) \cdot \PP{\Binomial(\lfloor\kappa\rfloor, \delta) < k^*_{\kappa,\delta}} + a^*_{\kappa,\delta} \cdot \PP{\Binomial(\lfloor\kappa\rfloor, \delta)\leq k^*_{\kappa,\delta}}\\
	&=\alpha.
\end{align*}
In the above, the inequality of the next-to-last step is because $\delta^*_\eps > \delta$, and hence, the $\Binomial(\lfloor\kappa\rfloor, \delta)$ distribution is stochastically smaller than the distribution of $B \sim \Binomial(\lfloor\kappa\rfloor, \delta^*_\eps)$. The equality of the last step is a direct consequence of the definitions of $k^*_{\kappa,\delta}$ and of $a^*_{\kappa,\delta}$.

Finally, we calculate power in the case that $T^*_{\eps,\delta} = 1$, for the special case $\delta\leq 1-\alpha^{1/\lfloor\kappa\rfloor}$ or $\delta^*_\eps=0$. If $\delta \leq 1-\alpha^{1/\lfloor\kappa\rfloor}$, then $k^*_{\kappa,\delta} = 0$, and so we have
\begin{multline*}
	\PP{\widehat{\test}_{\eps,\delta} = 1}
	= a^*_{\kappa,\delta} \cdot \PP{B = 0}
	= a^*_{\kappa,\delta} \cdot \PP{\Binomial(\lfloor\kappa\rfloor, \delta^*_\eps) = 0}\\
	= \alpha \cdot \frac{\PP{\Binomial(\lfloor\kappa\rfloor, \delta^*_\eps) = 0}}{\PP{\Binomial(\lfloor\kappa\rfloor, \delta) = 0}}
	= \alpha \cdot \rbr{\frac{1-\delta^*_\eps}{1-\delta}}^{\lfloor\kappa\rfloor}
	= \cbr{\alpha \cdot \rbr{\frac{1-\delta^*_\eps}{1-\delta}}^{\lfloor\kappa\rfloor}} \wedge 1,
\end{multline*}
where the third equality holds by definition of $a^*_{\kappa,\delta}$, and the last equality holds since $\delta \leq 1-\alpha^{1/\lfloor\kappa\rfloor}$ implies $\alpha \cdot \tcbr{\trbr{1-\delta^*_\eps} / \trbr{1-\delta}}^{\lfloor\kappa\rfloor} \leq {\alpha} / {(1-\delta)^{\lfloor \kappa\rfloor}} \leq 1$. Otherwise, we have that $\delta^*_\eps=0$ but $\delta > 1-\alpha^{1/\lfloor\kappa\rfloor}$. This implies $B = 0$ almost surely and $k^*_{\kappa,\delta} > 0$. Therefore,
\[
	\PP{\widehat{\test}_{\eps,\delta} = 1}
	= \PP{B < k^*_{\kappa,\delta}} + a^*_{\kappa,\delta} \cdot \PP{B = k^*_{\kappa,\delta}}
	= 1
	= \cbr{\alpha \cdot \rbr{\frac{1-\delta^*_\eps}{1-\delta}}^{\lfloor\kappa\rfloor}} \wedge 1,
\]
where the last step holds since $\delta^*_\eps = 0$ and $\delta > 1-\alpha^{1/\lfloor\kappa\rfloor}$ implies $\alpha \cdot \tcbr{\trbr{1-\delta^*_\eps} / \trbr{1-\delta}}^{\lfloor\kappa\rfloor} = {\alpha} / {(1-\delta)^{\lfloor \kappa\rfloor}} > 1$.

\section{Proof of Theorem~\ref{thm:optimality}} \label{app:proof_thm:optimality}

To prove~\eqref{eqn:power_optimal}, by definition of $\kappa$, it is equivalent to show that
\begin{equation} \label{eqn:power_optimal_labeled}
	\PP{\widehat{\test}_{\eps,\delta}(\alg,\dataset_\ell,\dataset_u) = 1} \leq \alpha \cdot \rbr{\frac{1-\delta^*_\eps}{1-\delta}}^{N_\ell / n}
\end{equation}
and that 
\begin{equation} \label{eqn:power_optimal_unlabeled}
	\PP{\widehat{\test}_{\eps,\delta}(\alg,\dataset_\ell,\dataset_u) = 1} \leq \alpha \cdot \rbr{\frac{1-\delta^*_\eps}{1-\delta}}^{(N_\ell+N_u) / (n+1)}.
\end{equation}
The first inequality~\eqref{eqn:power_optimal_labeled} shows the bound on power that is due to the limited amount of labeled data, while the second inequality~\eqref{eqn:power_optimal_unlabeled} shows the bound on power that is due to the limited amount of unlabeled data (since the total available number of draws of $X$ is equal to $N_\ell+N_u$).

The proofs of~\eqref{eqn:power_optimal_labeled} and~\eqref{eqn:power_optimal_unlabeled} both follow the same general recipe. The key idea of each proof is to construct an algorithm $\alg'$ and a data distribution $\dgp'$ where $(\alg',\dgp',n)$ is not $(\eps,\delta)$-stable, but with $\alg'$ and $\dgp'$ sufficiently similar to $\alg$ and $\dgp$ so that $(\alg,\dgp,n)$ and $(\alg',\dgp',n)$ are difficult to distinguish using the limited available data. The constructions will follow a shared structure:

\begin{itemize}

\item We will construct $\alg'$ and $\dgp'$ so that
\begin{equation} \label{eqn:not_stable}
	\textnormal{$(\alg',\dgp',n)$ does not satisfy $(\eps,\delta)$-stability,}
\end{equation}
and therefore, since $\widehat{\test}_{\eps,\delta}$ satisfies the validity condition~\eqref{eqn:valid_hat_T}, we must have
\begin{equation} \label{eqn:typeI_alg_prime}
	\PP{\widehat{\test}_{\eps,\delta}(\alg', \dataset_\ell', \dataset_u') = 1} \leq \alpha,
\end{equation}
where $(\dataset_\ell', \dataset_u')$ are drawn from $\dgp'$ (i.e., $\dataset_\ell'$ consists of $N_\ell$ many labeled points drawn i.i.d.~from $\dgp'$, and $\dataset_u'$ consists of $N_u$ many unlabeled points drawn i.i.d.~from the marginal distribution $\dgp_X'$).

\item Distribution $\dgp'$ is constructed to be similar to $\dgp$, by defining it as a mixture model of the form
\[
	\dgp' = (1-c) \cdot \dgp + c \cdot \dgp_1
\]
for some other distribution $\dgp_1$ on $\R^d \times \R$.

With this definition, we can couple the processes of drawing data from $\dgp$ and drawing data from $\dgp'$. Specifically, suppose datasets $(\dataset_\ell, \dataset_u)$ are drawn from $\dgp$ and datasets $(\dataset_\ell', \dataset_u')$ are drawn from $\dgp'$. Then, by definition of $\dgp'$, we can construct a coupling of $(\dataset_\ell, \dataset_u)$ and $(\dataset_\ell', \dataset_u')$ such that
\begin{equation} \label{eqn:same_data_unlabeled}
	\PPst{(\dataset_\ell', \dataset_u') = (\dataset_\ell, \dataset_u)}{(\dataset_\ell, \dataset_u)} \geq (1-c)^{N_\ell+N_u}.
\end{equation}
Moreover, if $\dgp_1$ and $\dgp$ have the same marginal distribution of $X$, then the two distributions differ only on the labeled samples, and so we will instead have
\begin{equation} \label{eqn:same_data_labeled}
	\PPst{(\dataset_\ell', \dataset_u') = (\dataset_\ell, \dataset_u)}{(\dataset_\ell, \dataset_u)} \geq (1-c)^{N_\ell}.
\end{equation}

\item Algorithm $\alg'$ is constructed to be similar to $\alg$, in the following sense. For datasets $(\dataset_\ell, \dataset_u)$ drawn from $\dgp$, we will design $\alg'$ to satisfy
\begin{equation} \label{eqn:same_T_hat}
	\PP{\widehat{\test}_{\eps,\delta}(\alg, \dataset_\ell, \dataset_u) = \widehat{\test}_{\eps,\delta}(\alg', \dataset_\ell, \dataset_u)} = 1,
\end{equation}
i.e., running the test $\widehat{\test}_{\eps,\delta}$ using algorithm $\alg$ or using algorithm $\alg'$ will result in the same answer almost surely
when $\widehat{T}_{\eps,\delta}$ is computed using the same randomization terms $\zeta^{(1)}, \zeta^{(2)}, \dots, \zeta$ (see Definition~\ref{def:black_box_test}). (On the other hand, $\alg'$ might not yield the same result as $\alg$ for some adversarial choices of the data that have zero probability under $\dgp$.)

Combining~\eqref{eqn:same_T_hat} with~\eqref{eqn:typeI_alg_prime} and~\eqref{eqn:same_data_unlabeled} yields
\begin{align*}
	\alpha
	&\geq \PP{\widehat{\test}_{\eps,\delta}(\alg', \dataset_\ell', \dataset_u') = 1}\\
	&\geq \PP{\widehat{\test}_{\eps,\delta}(\alg', \dataset_\ell, \dataset_u) = 1,\ (\dataset_\ell', \dataset_u') = (\dataset_\ell,\dataset_u)}\\
	&= \PP{\widehat{\test}_{\eps,\delta}(\alg', \dataset_\ell, \dataset_u) = 1} \cdot \PPst{(\dataset_\ell', \dataset_u') = (\dataset_\ell,\dataset_u)}{\widehat{\test}_{\eps,\delta}(\alg', \dataset_\ell, \dataset_u) = 1}\\
	&\geq \PP{\widehat{\test}_{\eps,\delta}(\alg, \dataset_\ell, \dataset_u) = 1} \cdot (1-c)^{N_\ell+N_u}.
\end{align*}
Here, the last step holds since 
\begin{multline*}
	\PPst{(\dataset_\ell',\dataset_u') = (\dataset_\ell, \dataset_u)}{\widehat{\test}_{\eps,\delta}(\alg', \dataset_\ell, \dataset_u) = 1}\\
	\begin{aligned}[t]
	&= \EEst{\PPst{(\dataset_\ell', \dataset_u') = (\dataset_\ell, \dataset_u)}{(\dataset_\ell, \dataset_u), \zeta^{(1)}, \zeta^{(2)}, \dots, \zeta}}{\widehat{\test}_{\eps,\delta}(\alg', \dataset_\ell, \dataset_u) = 1}\\
	&= \EEst{\PPst{(\dataset_\ell', \dataset_u') = (\dataset_\ell, \dataset_u)}{(\dataset_\ell, \dataset_u)}}{\widehat{\test}_{\eps,\delta}(\alg', \dataset_\ell, \dataset_u) = 1}\\
	&\geq \EEst{(1-c)^{N_\ell+N_u}}{\widehat{\test}_{\eps,\delta}(\alg', \dataset_\ell, \dataset_u) = 1}\\
	&= (1-c)^{N_\ell+N_u},
	\end{aligned}
\end{multline*}
where the second step holds since $(\dataset_\ell', \dataset_u')$ is independent of $\zeta^{(1)}, \zeta^{(2)}, \dots, \zeta$ conditional on $(\dataset_\ell,\dataset_u)$ by construction, while the third step applies~\eqref{eqn:same_data_unlabeled}. Rearranging terms in our work above, we have therefore established that
\begin{equation} \label{eqn:compare_power_unlabeled}
	\PP{\widehat{\test}_{\eps,\delta}(\alg, \dataset_\ell, \dataset_u) = 1}\leq \alpha \cdot (1-c)^{-(N_\ell+N_u)}.
\end{equation}
For the case that $\dgp_1$ and $\dgp$ have the same marginal distribution of $X$, combining~\eqref{eqn:same_T_hat} with~\eqref{eqn:typeI_alg_prime} and~\eqref{eqn:same_data_labeled} yields 
\[
	\alpha \geq \PP{\widehat{\test}_{\eps,\delta}(\alg', \dataset_\ell', \dataset_u') = 1}
	\geq \PP{\widehat{\test}_{\eps,\delta}(\alg, \dataset_\ell, \dataset_u) = 1} \cdot (1-c)^{N_\ell}
\]
with an analogous calculation, and so for this case, we have established that
\begin{equation} \label{eqn:compare_power_labeled}
	\PP{\widehat{\test}_{\eps,\delta}(\alg, \dataset_\ell, \dataset_u) = 1}\leq \alpha \cdot (1-c)^{-N_\ell}.
\end{equation}

\end{itemize}

Finally, to prove the bound~\eqref{eqn:power_optimal_labeled}, by examining~\eqref{eqn:compare_power_labeled} we see that it suffices to prove that
\begin{multline} \label{eqn:power_optimal_labeled_to_show}
	\textnormal{For any $c > 1 - \rbr{\frac{1-\delta}{1-\delta^*_\eps}}^{1/n}$, we can choose $\dgp_1$ and $\alg'$ such that}\\
	\textnormal{$(\dgp_1)_X = \dgp_X$ and both~\eqref{eqn:not_stable} and~\eqref{eqn:same_T_hat} hold.}
\end{multline}
Similarly, to prove the bound~\eqref{eqn:power_optimal_unlabeled}, by examining~\eqref{eqn:compare_power_unlabeled} we see that it suffices to prove that
\begin{multline} \label{eqn:power_optimal_unlabeled_to_show}
	\textnormal{For any $c > 1-\rbr{\frac{1-\delta}{1-\delta^*_\eps}}^{1/(n+1)}$, we can choose $\dgp_1$ and $\alg'$ such that}\\
	\textnormal{both~\eqref{eqn:not_stable} and~\eqref{eqn:same_T_hat} hold.}
\end{multline}
To complete the proof, we will now give explicit constructions of $\dgp_1$ and $\alg'$ to verify each of these claims.

\subsection{Proof of~\eqref{eqn:power_optimal_labeled_to_show}} \label{app:proof_thm:optimality_labeled}

Consider the datasets $(\dataset_\ell, \dataset_u)$ drawn from $\dgp$, and the datasets $(\dataset_\ell^{(1)}, \dataset_u^{(1)}), (\dataset_\ell^{(2)}, \dataset_u^{(2)}), \dots$ generated by running the test $\widehat{\test}_{\eps,\delta}$ on input datasets $(\dataset_\ell, \dataset_u)$ using algorithm $\alg$. Write 
\[
	\dataset_\ell = \big((X_1,Y_1), \dots, (X_{N_\ell},Y_{N_\ell})\big), \quad \dataset_u = \big(X_{N_\ell+1}, \dots, X_{N_\ell+N_u}\big),
\]
and for each $r \geq 1$ write
\[
	\dataset_\ell^{(r)} = \big((X^{(r)}_1,Y^{(r)}_1), \dots, (X^{(r)}_{N^{(r)}_\ell},Y^{(r)}_{N^{(r)}_\ell})\big), \quad \dataset_u^{(r)} = \big(X^{(r)}_{N^{(r)}_\ell+1}, \dots, X^{(r)}_{N^{(r)}_\ell+N^{(r)}_u}\big).
\]
(Note that $N^{(r)}_\ell, N^{(r)}_u$ might be random, but are required to be finite almost surely.)

Define 
\[
	\Ycal = \cbr{Y_1, \dots, Y_{N_\ell}, Y^{(1)}_1, \dots, Y^{(1)}_{N^{(1)}_\ell}, Y^{(2)}_1, \dots} \subset \R,
\]
the set of all $Y$ values that are observed at any point when running the test. We claim that there exists some $y_* \in \R$ such that 
\begin{equation} \label{eqn:y_star}
	\PP{y_* \in \Ycal} = 0.
\end{equation}
To see why, consider any $y \in \R$. Then, by the union bound,
\[
	\PP{y \in \Ycal} \leq \sum_{i = 1}^n \PP{Y_i = y} + \sum_{r \geq 1}\sum_{i \geq 1} \PP{N^{(r)}_\ell \geq i \textnormal{ and } Y^{(r)}_i = y}.
\]
Clearly, for each term in the sum, there are at most countably infinitely many values $y \in \R$ for which the corresponding probability term is positive. Therefore in total, there are only countably many $y \in \R$ for which the probability $\tPP{y \in \Ycal}$ is positive. This implies that some $y_*$ satisfying~\eqref{eqn:y_star} must exist.

Next, we define $\dgp' = (1-c) \cdot \dgp + c \cdot \dgp_1$ for
\[
	\dgp_1 = \dgp_X \times \delta_{y_*},
\]
where $\delta_{y_*}$ is the point mass at $y_*$. In other words, $\dgp_1$ is defined by drawing $X$ from $\dgp_X$ (the marginal distribution of $X$ under $\dgp$), and then setting $Y = y_*$ deterministically. Clearly, $\dgp_1$ and $\dgp$ have the same marginal distribution of $X$ by definition.

Next, we define the algorithm $\alg'$. For a labeled dataset $\dataset = \big((x_1,y_1), \dots, (x_m,y_m)\big)$, let $|\dataset| = m$ be its cardinality, and let $\dataset_{-m} = \big((x_1,y_1), \dots, (x_{m-1},y_{m-1})\big)$ be the dataset with the last point removed, and for any $\sigma \in \mathcal{S}_m$ (where $\mathcal{S}_m$ is the set of permutations on $\{1, \dots, m\}$), let $\dataset_\sigma$ denote the dataset permuted via $\sigma$, i.e., $\big((x_{\sigma(1)},y_{\sigma(1)}),\dots,(x_{\sigma(m)},y_{\sigma(m)})\big)$. $\alg'[\dataset; \xi]$ is then defined as follows. For any test point $x \in \R^d$, the outputted prediction is given by
\[
	\alg'\sbr{\dataset; \xi}(x) =
	\begin{cases}
	\alg_1\sbr{\dataset; \xi}(x) & \textnormal{ if $|\dataset| = n$ and $\sum_{i = 1}^n \One{y_i = y_*} \geq 1$,}\\
	\algof{\dataset; \xi}(x) & \textnormal{ otherwise},
	\end{cases}
\]
where
\begin{equation} \label{eqn:define_Alg1_labeled}
	\alg_1\sbr{\dataset; \xi}(x) = 1 + \eps + \max_{\sigma\in\mathcal{S}_m}\algof{(\dataset_\sigma)_{-m}; \xi}(x).
\end{equation}
This ensures that if $y_* \notin \dataset$ (i.e., the label $y_*$ is not observed in the data), then the two algorithms $\alg$ and $\alg'$ yield the same output. On the other hand, if the label $y_*$ is observed at least once in a sample of size $n$, then $\alg'$ returns a perturbed output. (Maximizing over permutations $\sigma$ ensures that $\alg_1$, and therefore $\alg'$, is a symmetric algorithm.)

Let $A_1, \dots, A_n \iidsim \operatorname{Bernoulli}(c)$ be drawn independently from the data $(X_i,Y_i) \iidsim \dgp$ and the randomization term $\xi \sim \Uniform[0,1]$. Then, defining
\[
	(X'_i,Y'_i) = \begin{cases}(X_i,Y_i) & \text{ if } A_i = 0,\\ (X_i,y_*) & \text{ if } A_i = 1,\end{cases}
\]
for $i = 1, \dots, n$ yields $n$ i.i.d.~draws from $\dgp'$. Defining
\[
	\muh'_n = \alg'\sbr{(X_1',Y_1'), \dots, (X_n',Y_n'); \xi}, \quad \muh'_{n-1} = \alg'\sbr{(X_1',Y_1'), \dots, (X_{n-1}',Y_{n-1}'); \xi},
\]
we see that by definition of $\alg'$, if $\sum_{i = 1}^n \tOne{Y_i' = y_*} \geq 1$ then we must have $|\muh'_n(X_{n+1}') - \muh'_{n-1}(X_{n+1}')| \geq 1+\eps > \eps$, i.e.,
\[
	\PP{\abr{\muh'_n(X_{n+1}') - \muh'_{n-1}(X_{n+1}')} > \eps\ \left\vert\ \sum_{i = 1}^n \One{Y_i' = y_*} \geq 1 \right.} = 1.
\]
Moreover, by definition of $\dgp'$ together with the fact that $\tPP{Y = y_*} = 0$ for $Y \sim \dgp$ (since $\tPP{y_* \in \Ycal} = 0$), we see that
\[
	\PP{\sum_{i = 1}^n \One{Y_i' = y_*} \geq 1}
	= \PP{\sum_{i = 1}^n \One{A_i = 1} \geq 1}
	= 1-(1-c)^n.
\]
Conversely, $\sum_{i=1}^n \tOne{Y_i'=y_*}=0$ holds if and only if $\sum_{i=1}^n \tOne{A_i=1} = 0$, and in this setting, by definition of $\dgp'$ and $\alg'$, we have
\begin{multline*}
	\PP{\abr{\muh'_n(X_{n+1}') - \muh'_{n-1}(X_{n+1}')} > \eps,\ \sum_{i = 1}^n \One{Y_i' = y_*} = 0}\\
	\begin{aligned}[t]
	&=\PP{\abr{\muh_n(X_{n+1}) - \muh_{n-1}(X_{n+1})} > \eps,\ \sum_{i = 1}^n \One{A_i = 1} = 0}\\
	&=\PP{\abr{\muh_n(X_{n+1}) - \muh_{n-1}(X_{n+1})} > \eps} \PP{\sum_{i = 1}^n \One{A_i = 1} = 0}
	=\delta^*_\eps (1-c)^n,
	\end{aligned}
\end{multline*}
where in the last step, we have used the definition of $\delta^*_\eps$, while the next-to-last step uses the fact that $A_1, \dots, A_n$ are independent from $(X_1,Y_1), \dots, (X_n,Y_n), \xi$ by construction, and are therefore independent from $\muh_n(X_{n+1}), \muh_{n-1}(X_{n+1})$. Combining these calculations, we have
\begin{align*}
&	\PP{\abr{\muh'_n(X_{n+1}') - \muh'_{n-1}(X_{n+1}')} > \eps}\\
	&=
	\begin{multlined}[t]
	\PP{\abr{\muh'_n(X_{n+1}') - \muh'_{n-1}(X_{n+1}')} > \eps,\ \sum_{i = 1}^n \One{Y_i' = y_*} \geq 1}\\
	+ \PP{\abr{\muh'_n(X_{n+1}') - \muh'_{n-1}(X_{n+1}')} > \eps,\ \sum_{i = 1}^n \One{Y_i' = y_*} = 0}
	\end{multlined}\\
	&= 1 - (1-c)^n + \delta^*_\eps (1-c)^n
	= 1 - (1-\delta^*_\eps) (1-c)^n
	> \delta,
\end{align*}
where the last step holds by our assumption on $c$. Therefore, $(\alg',\dgp',n)$ is not $(\eps,\delta)$-stable, which verifies~\eqref{eqn:not_stable}.

Finally, we verify~\eqref{eqn:same_T_hat}, i.e., we need to check that $\widehat{\test}_{\eps,\delta}$ will return the same answer almost surely using $\alg'$ as using $\alg$ when the input data is drawn from $\dgp$. By Definition~\ref{def:black_box_test}, we know that
\[
	\widehat{\test}_{\eps,\delta}(\alg, \dataset_\ell, \dataset_u) = g\sbr{\dataset_\ell, \dataset_u, \Big(\dataset_\ell^{(r)}\Big)_{r \geq 1}, \Big(\dataset_u^{(r)}\Big)_{r \geq 1}, \rbr{\Yhat^{(r)}}_{r \geq 1}, \Big(\zeta^{(r)}\Big)_{r \geq 1}, \Big(\xi^{(r)}\Big)_{r \geq 1}, \zeta}
\]
for some function $g$, where for each $r \geq 1$,
\[
	\rbr{\dataset_\ell^{(r)}, \dataset_u^{(r)}, \xi^{(r)}} = f^{(r)}\sbr{\dataset_\ell, \dataset_u, \Big(\dataset_\ell^{(s)}\Big)_{s = 1}^{r-1}, \Big(\dataset_u^{(s)}\Big)_{s = 1}^{r - 1}, \rbr{\Yhat^{(s)}}_{s = 1}^{r - 1}, \Big(\zeta^{(s)}\Big)_{s = 1}^{r-1}, \Big(\xi^{(s)}\Big)_{s = 1}^{r-1}, \zeta^{(r)}}
\]
and
\[
 	\muh^{(r)} = \algof{\dataset_\ell^{(r)}; \xi^{(r)}}, \quad \Yhat^{(r)} = \muh^{(r)}\Big[\dataset_u^{(r)}\Big].
\]
Similarly, we have
\[
	\widehat{\test}_{\eps,\delta}(\alg', \dataset_\ell, \dataset_u) = g\sbr{\dataset_\ell, \dataset_u, \Big(\widetilde{\dataset}_\ell^{(r)}\Big)_{r \geq 1}, \Big(\widetilde{\dataset}_u^{(r)}\Big)_{r \geq 1}, \rbr{\Ytilde^{(r)}}_{r \geq 1}, \Big(\zeta^{(r)}\Big)_{r \geq 1}, \Big(\widetilde\xi^{(r)}\Big)_{r \geq 1}, \zeta}
\]
for the same function $g$ and the same randomization terms $\zeta^{(1)}, \zeta^{(2)}, \dots, \zeta \iidsim \Uniform [0,1]$, where for each $r \geq 1$,
\[
	\rbr{\widetilde{\dataset}_\ell^{(r)}, \widetilde{\dataset}_u^{(r)}, \widetilde\xi^{(r)}} = f^{(r)}\sbr{\dataset_\ell, \dataset_u, \Big(\widetilde{\dataset}_\ell^{(s)}\Big)_{s = 1}^{r-1}, \Big(\widetilde{\dataset}_u^{(s)}\Big)_{s = 1}^{r - 1}, \rbr{\Ytilde^{(s)}}_{s = 1}^{r - 1}, \Big(\zeta^{(s)}\Big)_{s = 1}^{r-1}, \Big(\widetilde\xi^{(s)}\Big)_{s = 1}^{r-1}, \zeta^{(r)}},
\]
and
\[
 	\mut^{(r)} = \algof{\widetilde{\dataset}_\ell^{(r)}; \widetilde\xi^{(r)}}, \quad \Ytilde^{(r)} = \mut^{(r)}\Big[\widetilde{\dataset}_u^{(r)}\Big].
\]
We will now verify that, almost surely, for all $r \geq 1$ it holds that
\begin{equation} \label{eqn:induction_m}
	\rbr{\dataset_\ell^{(r)}, \dataset_u^{(r)}, \xi^{(r)}} = \rbr{\widetilde{\dataset}_\ell^{(r)}, \widetilde{\dataset}_u^{(r)}, \widetilde\xi^{(r)}}, \quad \muh^{(r)} = \mut^{(r)}, \quad \Yhat^{(r)} = \Ytilde^{(r)}.
\end{equation}
First, consider $r = 1$. Then, we have
\[
	\rbr{\dataset_\ell^{(1)}, \dataset_u^{(1)}, \xi^{(1)}} = f^{(1)}\sbr{\dataset_\ell, \dataset_u, \zeta^{(1)}} = \rbr{\widetilde{\dataset}_\ell^{(1)}, \widetilde{\dataset}_u^{(1)}, \widetilde\xi^{(1)}},
\]
and therefore
\[
	\muh^{(1)} = \algof{\dataset_\ell^{(1)}; \xi^{(1)}} = \alg'\sbr{\dataset_\ell^{(1)}; \xi^{(1)}} = \alg'\sbr{\widetilde{\dataset}_\ell^{(1)}; \widetilde\xi^{(1)}} = \mut^{(1)},
\]
where, by definition of $\alg'$, the second equality holds on the event that $Y_i^{(1)} \neq y_*$ for all $i = 1, \dots, N_\ell^{(1)}$, and this event holds almost surely by definition of $y_*$. So, we have
\[
	\Yhat^{(1)} = \muh^{(1)}\Big[\dataset_u^{(1)}\Big] = \mut^{(1)}\Big[\widetilde{\dataset}_u^{(1)}\Big] = \Ytilde^{(1)},
\] 
almost surely, which verifies~\eqref{eqn:induction_m} for the case $r = 1$. Next, consider any $r' \geq 2$. Suppose that~\eqref{eqn:induction_m} holds for all $r = 1, \dots, r'-1$. Then,
\begin{multline*}
	\rbr{\dataset_\ell^{(r')},\dataset_u^{(r')}, \xi^{(r')}}\\
	= f^{(r')}\sbr{\dataset_\ell, \dataset_u, \Big(\dataset_\ell^{(s)}\Big)_{s = 1}^{r'-1}, \Big(\dataset_u^{(s)}\Big)_{s = 1}^{r' - 1}, \rbr{\Yhat^{(s)}}_{s = 1}^{r' - 1}, \Big(\zeta^{(s)}\Big)_{s = 1}^{r'-1}, \Big(\xi^{(s)}\Big)_{s = 1}^{r'-1}, \zeta^{(r')}}\\
	= f^{(r')}\sbr{\dataset_\ell, \dataset_u, \Big(\widetilde{\dataset}_\ell^{(s)}\Big)_{s = 1}^{r'-1}, \Big(\widetilde{\dataset}_u^{(s)}\Big)_{s = 1}^{r' - 1}, \rbr{\Ytilde^{(s)}}_{s = 1}^{r' - 1}, \Big(\zeta^{(s)}\Big)_{s = 1}^{r'-1}, \Big(\widetilde\xi^{(s)}\Big)_{s = 1}^{r'-1}, \zeta^{(r')}}\\
	= \rbr{\widetilde{\dataset}_\ell^{(r')}, \widetilde{\dataset}_u^{(r')}, \widetilde\xi^{(r')}},
\end{multline*}
where the second equality holds almost surely by induction. Therefore we have
\[
	\muh^{(r')} = \algof{\dataset_\ell^{(r')}; \xi^{(r')}} = \alg'\sbr{\dataset_\ell^{(r')}; \xi^{(r')}} = \alg'\sbr{\widetilde{\dataset}_\ell^{(r')}; \widetilde\xi^{(r')}} = \mut^{(r')},
\]
where, by definition of $\alg'$, the second equality holds on the event that $Y_i^{(r')} \neq y_*$ for all $i = 1, \dots, N_\ell^{(r')}$, and this event holds almost surely by definition of $y_*$. So, we have
\[
	\Yhat^{(r')} = \muh^{(r')}\Big[\dataset_u^{(r')}\Big] = \mut^{(r')}\Big[\widetilde{\dataset}_u^{(r')}\Big] = \Ytilde^{(r')},
\] 
almost surely. Therefore,~\eqref{eqn:induction_m} holds almost surely for $r = r'$. By induction, we see that~\eqref{eqn:induction_m} holds almost surely for all $r \geq 1$, and so 
\begin{multline*}
	\widehat{\test}_{\eps,\delta}(\alg, \dataset_\ell, \dataset_u)\\
	= g\sbr{\dataset_\ell, \dataset_u, \Big(\dataset_\ell^{(r)}\Big)_{r \geq 1}, \Big(\dataset_u^{(r)}\Big)_{r \geq 1}, \rbr{\Yhat^{(r)}}_{r \geq 1}, \Big(\zeta^{(r)}\Big)_{r \geq 1}, \Big({\xi^{(r)}}\Big)_{r \geq 1}, \zeta}\\
	= g\sbr{\dataset_\ell, \dataset_u, \Big(\widetilde{\dataset}_\ell^{(r)}\Big)_{r \geq 1}, \Big(\widetilde{\dataset}_u^{(r)}\Big)_{r \geq 1}, \rbr{\Ytilde^{(r)}}_{r \geq 1}, \Big(\zeta^{(r)}\Big)_{r \geq 1}, \Big(\widetilde\xi^{(r)}\Big)_{r \geq 1}, \zeta}\\
	= \widehat{\test}_{\eps,\delta}(\alg', \dataset_\ell, \dataset_u) 
\end{multline*}
holds almost surely. This proves the claim~\eqref{eqn:same_T_hat}, and thus completes our proof of~\eqref{eqn:power_optimal_labeled_to_show}.

\subsection{Proof of~\eqref{eqn:power_optimal_unlabeled_to_show}}

Define
\[
	\Xcal = \cbr{X_1, \dots, X_{N_\ell+N_u}, X^{(1)}_1, \dots,X^{(1)}_{N^{(1)}_\ell+N^{(1)}_u}, X^{(2)}_1, \dots} \subset \R^d,
\]
the set of all $X$ values that are observed at any point when running the test. (Here, the $X_i$'s and $X^{(k)}_i$'s are defined as in Section~\ref{app:proof_thm:optimality_labeled}). Then, there exists some $x_* \in \R$ such that 
\[
	\PP{x_* \in \Xcal} = 0.
\]
The proof of this claim is analogous to the proof of~\eqref{eqn:y_star} in Section~\ref{app:proof_thm:optimality_labeled}, so we do not repeat the details here. Define $\dgp' = (1-c) \cdot \dgp + c \cdot \dgp_1$ for
\[
	\dgp_1 = \delta_{x_*}\times P_Y,
\]
where $\delta_{x_*}$ is the point mass at $x_*$. In other words, $\dgp_1$ is defined by drawing $Y$ from $\dgp_Y$ (the marginal distribution of $Y$ under $\dgp$), and then setting $X = x_*$ deterministically. 

Next, we define the algorithm $\alg'$. For any labeled dataset $\dataset = \big((x_1,y_1), \dots, (x_m,y_m)\big)$ and any test point $x \in \R^d$, define
\[
	\alg'\sbr{\dataset; \xi}(x) =
	\begin{cases}
	\alg_1\sbr{\dataset; \xi}(x) &\textnormal{ if $|\dataset| = n$, and $\sum_{i = 1}^n \One{x_i = x_*} \geq 1$ or $x = x_*$,}\\
	\algof{\dataset; \xi}(x) & \textnormal{ otherwise},
	\end{cases}
\]
where $\alg_1$ is defined as in~\eqref{eqn:define_Alg1_labeled} from before.

Now, we need to verify~\eqref{eqn:not_stable}, i.e., $(\alg',\dgp',n)$ is not $(\eps,\delta)$-stable. This proof is similar to the analogous proof in Section~\ref{app:proof_thm:optimality_labeled}. Let $A_1, \dots, A_n \iidsim \operatorname{Bernoulli}(c)$ be drawn independently from $(X_i,Y_i)\iidsim \dgp$ and $\xi \iidsim \Uniform[0,1]$. Then, defining
\[
	(X'_i,Y'_i) = \begin{cases}(X_i,Y_i) & \text{ if } A_i = 0,\\ (x_*,Y_i) & \text{ if } A_i = 1,\end{cases}
\]
for $i = 1, \dots, n$ yields $n$ i.i.d.~draws from $\dgp'$. First, we can observe that if $\sum_{i = 1}^{n+1} \tOne{X_i' = x_*} \geq 1$ then $|\muh'_n(X_{n+1}') - \muh'_{n-1}(X_{n+1}')| \geq 1+\eps > \eps$, i.e.,
\[
	\PP{\abr{\muh'_n(X_{n+1}') - \muh'_{n-1}(X_{n+1}')} > \eps\ \left\vert\ \sum_{i = 1}^{n+1} \One{X_i' = x_*} \geq 1 \right.} = 1.
\]
We can also calculate
\[
	\PP{\sum_{i = 1}^{n+1} \One{X_i' = x_*} \geq 1} = \PP{\sum_{i = 1}^{n+1} \One{A_i = 1} \geq 1 } = 1-(1-c)^{n+1}.
\]
Next, exactly as in Section~\ref{app:proof_thm:optimality_labeled}, we have
\begin{multline*}
	\PP{\abr{\muh'_n(X_{n+1}') - \muh'_{n-1}(X_{n+1}')} > \eps,\ \sum_{i = 1}^{n+1} \One{X_i' = x_*} = 0}\\
	\begin{aligned}[t]
	&= \PP{\abr{\muh_n(X_{n+1}) - \muh_{n-1}(X_{n+1})} > \eps,\ \sum_{i = 1}^{n+1} \One{A_i = 1} = 0}\\
	&= \PP{\abr{\muh_n(X_{n+1}) - \muh_{n-1}(X_{n+1})} > \eps} \PP{\sum_{i = 1}^{n+1} \One{A_i=1} = 0}
	= \delta^*_\eps (1-c)^{n+1}.
	\end{aligned}
\end{multline*}
Combining these calculations, we obtain
\begin{align*}
&	\PP{\abr{\muh'_n(X_{n+1}') - \muh'_{n-1}(X_{n+1}')} > \eps}\\
	&= 
	\PP{\abr{\muh'_n(X_{n+1}') - \muh'_{n-1}(X_{n+1}')} > \eps,\ \sum_{i = 1}^{n+1} \One{X_i' = x_*} \geq 1}\\
	&\hspace{1in}+ \PP{\abr{\muh'_n(X_{n+1}') - \muh'_{n-1}(X_{n+1}')} > \eps,\ \sum_{i = 1}^{n+1} \One{X_i' = x_*} = 0}\\
	&= 1 - (1-c)^{n+1} + \delta^*_\eps (1-c)^{n+1}
	= 1-\rbr{1-\delta^*_\eps} (1-c)^{n+1}
	> \delta,
\end{align*}
where the last step holds by our assumption on $c$. Therefore, $(\alg',\dgp',n)$ is not $(\eps,\delta)$-stable, which verifies~\eqref{eqn:not_stable} for this setting.

It remains to verify~\eqref{eqn:same_T_hat}. The proof of~\eqref{eqn:same_T_hat} for this setting is essentially identical to the analogous statement in Section~\ref{app:proof_thm:optimality_labeled}, so we omit it here.

\section{Proof of Proposition~\ref{prop:T_delta_eps}} \label{app:proof_prop:T_delta_eps}

Assume $T^*_{\eps,\delta} = 0$, i.e., $\delta < \delta^*_\eps$ and $\eps < \eps^*_\delta$. To prove part (a), we have
\begin{align*}
	\PP{\widehat{\test}_{\eps,\delta} = 1}
	&= \PP{\widehat{\delta}_\eps \leq \delta} \textnormal{\quad by definition of the test $\widehat{\test}_{\eps,\delta}$}\\
	&\leq \PP{\widehat{\delta}_\eps < \delta^*_\eps} \textnormal{\quad since $\delta < \delta^*_\eps$}\\
	&\leq \alpha \textnormal{\quad by assumption-free validity of $\widehat{\delta}_\eps$~\eqref{eqn:valid_hat_delta}.}
\end{align*}
Similarly, to prove (b), we have
\begin{align*}
	\PP{\widehat{\test}_{\eps,\delta} = 1}
	&= \PP{\widehat{\eps}_\delta \leq \eps} \textnormal{\quad by definition of the test $\widehat{\test}_{\eps,\delta}$}\\
	&\leq \PP{\widehat{\eps}_\delta < \eps^*_\delta} \textnormal{\quad since $\eps < \eps^*_\delta$}\\
	&\leq \alpha \textnormal{\quad by assumption-free validity of $\widehat{\eps}_\delta$~\eqref{eqn:valid_hat_eps}.}
\end{align*}
Next, we prove part (c). We have
\begin{align*}
	\PP{\widehat{\delta}_\eps < \delta^*_\eps} 
	&=\sup_{\delta < \delta^*_\eps} \PP{\widehat{\delta}_\eps < \delta}\\
	&\leq \sup_{\delta < \delta^*_\eps} \PP{\widehat{\test}_{\eps,\delta} = 1} \textnormal{\quad since $\widehat{\delta}_\eps < \delta$ implies $\widehat{\test}_{\eps,\delta} = 1$ by definition of $\widehat{\delta}_\eps$}\\
	 &\leq \sup_{\delta < \delta^*_\eps}\alpha\textnormal{\quad by assumption-free validity of $\widehat{\test}_{\eps,\delta}$~\eqref{eqn:valid_hat_T}}\\
	&= \alpha.
\end{align*}
Similarly, we verify part (d). We have
\begin{align*}
	\PP{\widehat{\eps}_\delta < \eps^*_\delta}
	&=\sup_{\eps < \eps^*_\delta} \PP{\widehat{\eps}_\delta < \eps}\\
	&\leq \sup_{\eps < \eps^*_\delta} \PP{\widehat{\test}_{\eps,\delta} = 1} \textnormal{\quad since $\widehat{\eps}_\delta < \eps$ implies $\widehat{\test}_{\eps,\delta} = 1$ by definition of $\widehat{\eps}_\delta$}\\
	 &\leq \sup_{\eps < \eps^*_\delta}\alpha\textnormal{\quad by assumption-free validity of $\widehat{\test}_{\eps,\delta}$~\eqref{eqn:valid_hat_T}}\\
	&= \alpha.
\end{align*}

\section{Proofs of Theorems~\ref{thm:optimality_transparent} and~\ref{thm:optimality_constrained}} \label{app:proofs_transparent}

First, we observe that Theorem~\ref{thm:optimality_transparent} is simply a special case of Theorem~\ref{thm:optimality_constrained} obtained by taking $\modelclass = \modelspace$. In particular, the condition~\eqref{eqn:modelclass_assume} is satisfied because the function $(\mu_1, \dots, \mu_m) \mapsto \mu$, where $\mu(x) = 1 + t + \max_{i=1,\dots,m} \mu_i(x)$, is an example of $M$.

Now, we prove Theorem~\ref{thm:optimality_constrained}. At a high level, the proof is quite similar to the proof of Theorem~\ref{thm:optimality}. However, because $\widehat{\test}_{\eps,\delta}$ is now only guaranteed to satisfy the weaker validity \eqref{eqn:valid_hat_T_modelclass} (i.e., validity only with respect to algorithms in $\algspace_{\modelclass}$), our construction $\alg'$ must also be a member of $\algspace_{\modelclass}$. Specifically, following the notation defined in the proof of Theorem~\ref{thm:optimality}, we need to verify:
 \begin{multline} \label{eqn:power_optimal_labeled_to_show_transparent}
 	\textnormal{For any $c > 1 - \rbr{\frac{1-\delta}{1-\delta^*_\eps}}^{1/n}$, we can choose $\dgp_1$ and $\alg'\in\algspace_{\modelclass}$ such that}\\
	\textnormal{$(\dgp_1)_X = \dgp_X$ and both~\eqref{eqn:not_stable} and~\eqref{eqn:same_T_hat} hold.}
\end{multline}
This is the same as the analogous claim~\eqref{eqn:power_optimal_labeled_to_show} appearing in the proof of Theorem~\ref{thm:optimality},
except that $\alg'$ is now constrained to lie in $\algspace_{\modelclass}$.

As in the proof of~\eqref{eqn:power_optimal_labeled_to_show} (in Section~\ref{app:proof_thm:optimality_labeled}), we define $\dgp' = (1-c) \cdot \dgp + c \cdot \dgp_1$ for
\[
	\dgp_1 = \dgp_X \times \delta_{y_*},
\]
where $\delta_{y_*}$ is the point mass at $y_*$, for some $y_*$ chosen to satisfy~\eqref{eqn:y_star} as before. Clearly, $\dgp_1$ and $\dgp$ have the same marginal distribution of $X$ by definition.

We now modify our construction of $\alg'$ so that $\alg' \in \algspace_{\modelclass}$. Note that since $c > 1 - \{(1-\delta) / (1-\delta^*_\eps)\}^{1/n}$, we have that $c > 1 - \{(1-\delta-a) / (1-\delta^*_\eps-a)\}^{1/n}$ for sufficiently small $a > 0$. Fix any such $a\in(0,1-\delta^*_\eps)$. Choose some finite $L$ such that
\[
	\Pp{P_X}{\|X\|_2 > L} \leq a.
\]
The fitted model $\alg'[\dataset; \xi]$ is then defined as follows:
\[
	\alg'\sbr{\dataset; \xi}(x) =
	\begin{cases}
	\alg_1\sbr{\dataset; \xi}(x) & \textnormal{ if $|\dataset| = n$ and $\sum_{i = 1}^n \One{y_i = y_*} \geq 1$,}\\
	\algof{\dataset; \xi}(x) & \textnormal{ otherwise},
	\end{cases}
\]
where, for a dataset $\dataset$ of size $|\dataset| = m$,
\[	\alg_1\sbr{\dataset; \xi}(x) = M\rbr{\alg\sbr{(\dataset_\sigma)_{-m}; \xi}: \sigma \in \mathcal{S}_m}
\]
for the function $M$ defined as in~\eqref{eqn:modelclass_assume}. By construction, $\alg'$ is symmetric, and satisfies $\alg' \in \algspace_{\modelclass}$.

We now verify that this new construction of $(\alg', \dgp', n)$ is not $(\eps,\delta)$-stable. Let $A_1, \dots, A_n \iidsim \operatorname{Bernoulli}(c)$ be drawn independently from $(X_i,Y_i) \iidsim P$ and from $\xi \sim \Uniform[0,1]$. Then, defining
\[
	(X'_i,Y'_i) = \begin{cases}(X_i,Y_i) & \text{ if } A_i = 0,\\ (X_i,y_*) & \text{ if } A_i = 1,\end{cases}
\]
for $i = 1, \dots, n$ yields $n$ i.i.d.~draws from $\dgp'$. Defining
\[
	\muh'_n = \alg'\sbr{(X_1',Y_1'), \dots, (X_n',Y_n'); \xi}, \quad \muh'_{n-1} = \alg'\sbr{(X_1',Y_1'), \dots, (X_{n-1}',Y_{n-1}'); \xi},
\]
we see that by definition of $\alg'$, if $\sum_{i = 1}^n \tOne{Y_i' = y_*} \geq 1$ then $|\muh'_n(X_{n+1}') - \muh'_{n-1}(X_{n+1}')| > \eps$ on the event $\|X_{n+1}\|_2 \leq L$. Since the event $\|X_{n+1}\|_2 \leq L$ holds with probability at least $1-a$, and is independent from the $A_i$'s, while $\sum_{i=1}^n \One{Y_i'=y_*}\geq 1$ holds if and only if $\sum_{i=1}^n \One{A_i=1}\geq 1$, this implies
\[
	\PP{\abr{\muh'_n(X_{n+1}') - \muh'_{n-1}(X_{n+1}')} > \eps\ \left\vert\ \sum_{i = 1}^n \One{Y_i' = y_*} \geq 1 \right.} \geq 1-a,
\]
while
\[
	\PP{\sum_{i = 1}^n \One{Y_i'=y_*} \geq 1}
	= \PP{\sum_{i = 1}^n \One{A_i = 1} \geq 1}
	= 1-(1-c)^n
\]
holds as in the proof of Theorem~\ref{thm:optimality}. We can also calculate
\[
	\PP{\abr{\muh'_n(X_{n+1}') - \muh'_{n-1}(X_{n+1}')} > \eps,\ \sum_{i = 1}^n \One{Y_i' = y_*} = 0} 
	=\delta^*_\eps (1-c)^n
\]
exactly as in the proof of Theorem~\ref{thm:optimality}. Combining these calculations, we have
\begin{multline*}
	\PP{\abr{\muh'_n(X_{n+1}') - \muh'_{n-1}(X_{n+1}')} > \eps}\\
	\begin{aligned}[t]
	&= \begin{multlined}[t]\PP{\abr{\muh'_n(X_{n+1}') - \muh'_{n-1}(X_{n+1}')} > \eps,\ \sum_{i = 1}^n \One{Y_i' = y_*} \geq 1}\\
	+ \PP{\abr{\muh'_n(X_{n+1}') - \muh'_{n-1}(X_{n+1}')} > \eps,\ \sum_{i = 1}^n \One{Y_i' = y_*} = 0}
	\end{multlined}\\
	&\geq (1-a) \cbr{1 - (1-c)^n} + \delta^*_\eps (1-c)^n
	= (1-a) - \cbr{(1-a)-\delta^*_\eps} (1-c)^n
	> \delta,
	\end{aligned}
\end{multline*}
where the last step holds by our assumptions on $c$ and $a$. Therefore, $(\alg',\dgp',n)$ is not $(\eps,\delta)$-stable, i.e., we have verified that~\eqref{eqn:not_stable} holds. Finally, the proof of~\eqref{eqn:same_T_hat} is identical to the proof of the same claim in Theorem~\ref{thm:optimality}, and so we omit it here. This verifies~\eqref{eqn:power_optimal_labeled_to_show_transparent}, and thus completes the proof of the theorem.

\section{Algorithmic randomness and algorithmic stability} \label{app:general_stability}

In this section, we discuss other definitions of stability in the context of randomized algorithms---that is, algorithms where the fitted model $\muh_n = \algof{(X_1,Y_1), \dots, (X_n,Y_n); \xi}$ depends on the randomization term $\xi$ as well as the data itself. This can include methods that incorporate steps such as stochastic gradient descent, random initialization, bootstrapping, subsampling the data and/or the features, or randomized approximations. As noted by \citet{Elisseeff2005Stability}, the randomization term $\xi$ may have a large impact on the output. For example, for an algorithm $\alg$ that fits a model via a bootstrapping type procedure, if the dataset has a few extreme outliers, then even on the {\em same} dataset, different draws of $\xi$ might lead to very different fitted models $\muh_n$, depending on whether the outliers were included or excluded in the bootstrapped samples.

The alternative definitions we will consider here will differ in how they handle algorithmic randomness, which impacts the contexts in which each of these definitions is relevant. We recall that Definition~\ref{def:stability} compared the models 
\[
	\muh_n = \algof{(X_1,Y_1), \dots, (X_n,Y_n); \xi},\quad
	\muh_{n-1} = \algof{(X_1,Y_1), \dots, (X_{n-1},Y_{n-1}); \xi}
\]
which are fitted using the same value $\xi$ drawn from $\Uniform[0,1]$. For instance, this definition is of interest when the randomness in $\alg$ originates from the choice of the initial point of some optimization algorithm, and we want to measure the effect of removing a training data point while fixing the initialization. In contrast, in some settings we may want to compare {\em independent} runs of $\alg$ corresponding to independent draws of the randomization terms, e.g., because this randomization mechanism is unknown or inaccessible. In other words, we may want to ask whether the fitted models 
\[
	\muh_n = \algof{(X_1,Y_1), \dots, (X_n,Y_n); \xi},\quad
	\muh_{n-1} = \algof{(X_1,Y_1), \dots, (X_{n-1},Y_{n-1}); \xi'}
\]
lead to similar predictions on a new test point $X_{n+1}$, where the randomization noise terms $\xi, \xi'$ are now independent draws from the $\Uniform[0,1]$ distribution. In other settings, it may instead be more relevant to ask whether $\muh_n$ and $\muh_{n-1}$ can be ensured to lead to similar predictions using {\em some} coupling of the noise terms $\xi$ and $\xi'$. An example of a setting in which this definition is the most pertinent occurs when the randomization mechanism depends on $n$, e.g., when $\alg$ involves bootstrapping or subsampling of the given training dataset. For instance, to analyze how bagging affects stability, \citet{Elisseeff2005Stability} linked the randomization terms for $\muh_n$ and $\muh_{n-1}$ in a very specific way, defining $\muh_{n-1}$ as the model that ``reuses" the same bags as $\muh_n$ by erasing the training data point from each bag containing it before applying the base learning method.

Therefore, we are comparing three different notions of stability: first, our existing definition (coinciding with Definition~\ref{def:stability}),
\begin{multline} \label{eq:stability_original}
	\PP{\abr{\muh_n(X_{n+1}) - \muh_{n-1}(X_{n+1})} > \eps} \leq \delta \textnormal{ where } \muh_n = \algof{(X_1,Y_1), \dots, (X_n,Y_n); \xi}\\
	\textnormal{ and } \muh_{n-1} = \algof{(X_1,Y_1), \dots, (X_{n-1},Y_{n-1}); \xi} \textnormal{ for $\xi \sim \Uniform[0,1]$};
\end{multline}
second, the definition with independent $\xi,\xi'$,
\begin{multline} \label{eq:stability_independent}
	\PP{\abr{\muh_n(X_{n+1}) - \muh_{n-1}(X_{n+1})} > \eps} \leq \delta \textnormal{ where } \muh_n = \algof{(X_1,Y_1), \dots, (X_n,Y_n); \xi}\\
	\textnormal{ and } \muh_{n-1} = \algof{(X_1,Y_1), \dots, (X_{n-1},Y_{n-1}); \xi'} \textnormal{ for $\xi, \xi' \iidsim \Uniform[0,1]$};
\end{multline}
and finally, the definition over an arbitrary coupling,
\begin{multline} \label{eq:stability_coupling}
	\PP{\abr{\muh_n(X_{n+1}) - \muh_{n-1}(X_{n+1})} > \eps} \leq \delta \textnormal{ where } \muh_n = \algof{(X_1,Y_1), \dots, (X_n,Y_n); \xi}\\
	\textnormal{ and } \muh_{n-1} = \algof{(X_1,Y_1), \dots, (X_{n-1},Y_{n-1}); \xi'} \textnormal{ for $(\xi,\xi') \sim Q$,}\\
	\textnormal{ for some distribution $Q$ on $[0,1]^2$ with marginals that are $\Uniform[0,1]$}.
\end{multline}
To be precise, for this last definition, we are asking whether there exists {\em any} coupling $Q$ for which the probability bound holds.

As before, we let $ \test^*_{\eps,\delta}(\alg,\dgp,n) = 1$ denote that the triple $(\alg,\dgp,n)$ satisfies the original definition of stability given in~\eqref{eq:stability_original} (equivalent to Definition~\ref{def:stability}), while $\test^{\indep}_{\eps,\delta}(\alg,\dgp,n) = 1$ (respectively, $\test^{\textnormal{cpl}}_{\eps,\delta}(\alg,\dgp,n) = 1$) will denote that $(\alg,\dgp,n)$ satisfies the independence-based definition~\eqref{eq:stability_independent} (respectively, the coupling-based definition~\eqref{eq:stability_coupling}).

In particular, we can see that this last definition is a strict relaxation of the first two: for any $(\alg,\dgp,n)$, if the original stability definition~\eqref{eq:stability_original} holds then the coupling definition~\eqref{eq:stability_coupling} also holds (by taking $Q$ to be the distribution that draws $\xi \sim \Uniform[0,1]$ and sets $\xi'=\xi$), and if the independence-based stability definition~\eqref{eq:stability_independent} holds then again the coupling definition~\eqref{eq:stability_coupling} also holds (by taking $Q$ to be the product distribution, $\Uniform[0,1] \times \Uniform[0,1]$). In other words, we have shown that
\[
	\test^*_{\eps,\delta}(\alg,\dgp,n) = 1 \quad \textnormal{or} \quad \test^{\indep}_{\eps,\delta}(\alg,\dgp,n) = 1 \quad \implies \quad \test^{\textnormal{cpl}}_{\eps,\delta}(\alg,\dgp,n) = 1.
\]

There is also a partial relation between the original definition and the independence-based definition---if $(\alg,\dgp,n)$ satisfies the independence-based condition~\eqref{eq:stability_independent}, then the original definition~\eqref{eq:stability_original} holds with $(3\eps,3\delta)$ in place of $(\eps,\delta)$, because we have
\begin{align*}
	&\PP{\abr{\muh_n(\xi)(X_{n+1}) - \muh_{n-1}(\xi)(X_{n+1})}>3\eps}\\
	&\hspace{.5in} \leq \PP{\abr{\muh_n(\xi)(X_{n+1}) - \muh_{n-1}(\xi')(X_{n+1})} > \eps}\\
	&\hspace{1in} + \PP{\abr{\muh_n(\xi'')(X_{n+1}) - \muh_{n-1}(\xi')(X_{n+1})} > \eps}\\
	&\hspace{1.5in} + \PP{\abr{\muh_n(\xi'')(X_{n+1}) - \muh_{n-1}(\xi)(X_{n+1})} > \eps} \leq 3\delta,
\end{align*}
where $\xi, \xi', \xi'' \iidsim \Uniform[0,1]$, and where, e.g., $\muh_n(\xi)$ denotes $ \algof{(X_1,Y_1), \dots, (X_n,Y_n); \xi}$. This proves that
\[
	\test^{\indep}_{\eps,\delta}(\alg,\dgp,n) = 1 \quad \implies \quad \test^*_{3\eps,3\delta}(\alg,\dgp,n) = 1.
\]

Now, for a black-box testing procedure $\widehat{\test}_{\eps,\delta}$, we may want to ensure that the test has assumption-free validity as a test of any one of our different notions of stability---that is, while the original notion of validity requires
\begin{equation} \label{eqn:valid_hat_T_original}
	\textnormal{For all $(\alg,\dgp,n)$, if $\test^*_{\eps,\delta}(\alg,\dgp,n) = 0$ then $\PP{\widehat{\test}_{\eps,\delta} = 1} \leq \alpha$},
\end{equation}
we might instead wish to test stability with respect to independent $\xi,\xi'$ as in the stability definition~\eqref{eq:stability_independent},
\begin{equation} \label{eqn:valid_hat_T_independent}
	\textnormal{For all $(\alg,\dgp,n)$, if $\test^{\indep}_{\eps,\delta}(\alg,\dgp,n) = 0$ then $\PP{\widehat{\test}_{\eps,\delta} = 1} \leq \alpha$},
\end{equation}
or with respect to the coupling-based definition~\eqref{eq:stability_coupling}, 
\begin{equation} \label{eqn:valid_hat_T_coupling}
	\textnormal{For all $(\alg,\dgp,n)$, if $\test^{\textnormal{cpl}}_{\eps,\delta}(\alg,\dgp,n) = 0$ then $\PP{\widehat{\test}_{\eps,\delta} = 1} \leq \alpha$.}
\end{equation}
Our original result (Theorem~\ref{thm:optimality}) provided a bound on the power of any black-box test $\widehat{\test}_{\eps,\delta}$ that satisfied validity in its original form~\eqref{eqn:valid_hat_T_original}. Here, we give an analogous result for black-box tests satisfying the alternative assumption-free validity conditions~\eqref{eqn:valid_hat_T_independent} or~\eqref{eqn:valid_hat_T_coupling}.

\begin{theorem} \label{thm:optimality_general}
Fix any parameters $\eps \geq 0$ and $\delta \in [0,1)$, any desired error level $\alpha \in (0,1)$, and any integers $n \geq 2$ and $N_\ell, N_u \geq 0$. Let $\widehat{\test}_{\eps,\delta}$ be any black-box test as in Definition~\ref{def:black_box_test}. Let $\kappa=\kappa(n,N_\ell,N_u)$ be defined as in~\eqref{eqn:define_kappa}.

If $\widehat{\test}_{\eps,\delta}$ satisfies the assumption-free validity condition~\eqref{eqn:valid_hat_T_independent} at level $\alpha$, then for any $(\alg,\dgp,n)$ that satisfies the condition~\eqref{eq:stability_independent} (i.e., stability with respect to independent $\xi, \xi'$), the power of $\widehat{\test}_{\eps,\delta}$ is bounded as
\begin{equation} \label{eqn:power_optimal_independent}
	\PP{\widehat{\test}_{\eps,\delta} = 1} \leq \cbr{\alpha \cdot \rbr{\frac{1-\delta^{\indep}_\eps}{1-\delta}}^{\kappa}} \wedge 1.
\end{equation}

If instead $\widehat{\test}_{\eps,\delta}$ satisfies the assumption-free validity condition~\eqref{eqn:valid_hat_T_coupling} at level $\alpha$, then for any $(\alg,\dgp,n)$ that satisfies the condition~\eqref{eq:stability_coupling} (i.e., stability with respect to any coupled $\xi, \xi'$), the power of $\widehat{\test}_{\eps,\delta}$ is bounded as
\begin{equation} \label{eqn:power_optimal_coupling}
	\PP{\widehat{\test}_{\eps,\delta} = 1} \leq \cbr{\alpha \cdot \rbr{\frac{1-\delta^{\textnormal{cpl}}_\eps}{1-\delta}}^{\kappa}} \wedge 1.
\end{equation}
\end{theorem}

Here, $\delta^{\indep}_{\eps}$ and $\delta^{\textnormal{cpl}}_\eps$ are analogous to the term $\delta^*_\eps$ appearing in Theorem~\ref{thm:optimality}, with $\delta^{\indep}_{\eps}$ (respectively, $\delta^{\textnormal{cpl}}_{\eps}$) defined as the minimum value of $\delta$ such that $(\alg,\dgp,n)$ satisfies the independence-based stability condition~\eqref{eq:stability_independent} (respectively, the coupling-based stability condition~\eqref{eq:stability_coupling}). 
\begin{proof}[Proof of Theorem~\ref{thm:optimality_general}]
The proof of this theorem mostly parallels the proof of Theorem~\ref{thm:optimality}, with one modification: in place of our earlier definition~\eqref{eqn:define_Alg1_labeled} of the algorithm $\alg_1$, we instead define
\[
	\alg_1\sbr{\dataset; \xi}(x) = 1 + \eps + \max_{\sigma \in \mathcal{S}_m} q_{1-a} \rbr{\algof{(\dataset_\sigma)_{-m}; \xi'}(x)},
\]
where for any dataset $\dataset$ and any $x$, $q_{1-a}\trbr{\talgof{\dataset; \xi'}(x)}$ denotes the $(1-a)$-quantile of the distribution of $\talgof{\dataset; \xi'}(x)$ with respect to $\xi' \sim \Uniform[0,1]$, and where the constant $a \in (0,1-\delta^{\indep}_\eps)$ is chosen to be sufficiently small so that \[c > 1 - \rbr{(\frac{1-\delta-a}{1-\delta^{\indep}_\eps-a}}^{1/n} \textnormal{ or } c > 1 - \rbr{\frac{1-\delta-a}{1-\delta^{\indep}_\eps-a}}^{1/(n+1)}\] (for proving the analogous statement to~\eqref{eqn:power_optimal_labeled_to_show} or to~\eqref{eqn:power_optimal_unlabeled_to_show} when $\xi \indep \xi'$), or $a \in (0,1-\delta^{\textnormal{cpl}}_\eps)$ is chosen to be sufficiently small so that \[c > 1 - \rbr{\frac{1-\delta-a}{1-\delta^{\textnormal{cpl}}_\eps-a}}^{1/n} \textnormal{ or } c > 1 - \rbr{\frac{1-\delta-a}{1-\delta^{\textnormal{cpl}}_\eps-a}}^{1/(n+1)}\] (for proving the analogous statement to~\eqref{eqn:power_optimal_labeled_to_show} or to~\eqref{eqn:power_optimal_unlabeled_to_show} when $\xi, \xi'$ are arbitrarily coupled). With this change in place, the remainder of the proof is a straightforward modification of the proof of Theorem~\ref{thm:optimality}.
\end{proof}

Next, we comment on the role of the Binomial test in this broader setting. For the independence-based case, the Binomial test defined in Section~\ref{sec:simple_test} can be extended in a straightforward way---namely, by generating independent randomization terms $\xi^{(k)},\xi'{}^{(k)} \iidsim \Uniform[0,1]$ for each dataset $k = 1,\dots,\lfloor\kappa\rfloor$. The results of Theorem~\ref{thm:simple_test} are then exactly the same, except with $\delta^{\indep}_\eps$ in place of $\delta^*_\eps$; the proof is identical, and so we omit it here. Comparing to the power bound~\eqref{eqn:power_optimal_independent} in Theorem~\ref{thm:optimality_general}, this verifies that, for an integer $\kappa$, in the regimes $\delta \leq 1-\alpha^{1/\lfloor\kappa\rfloor}$ or $\delta^{\indep}_\eps = 0$, the Binomial test is again optimal. On the other hand, for the coupling-based definition~\eqref{eq:stability_coupling}, making an analogous statement is less straightforward since, to run a Binomial test that achieves optimal power, we would need to know the optimal coupling distribution $Q$ from which to draw $(\xi^{(k)},\xi'{}^{(k)})$. Therefore, we do not consider this extension here.

Finally, we comment on how the power bounds given in Theorem~\ref{thm:optimality_general} compare to the earlier Theorem~\ref{thm:optimality}.
Recall that we have shown that for any triple $(\alg,\dgp,n)$,
\[
	\test^{\textnormal{cpl}}_{\eps,\delta}(\alg,\dgp,n) = 0 \quad \implies \quad \test^*_{\eps,\delta}(\alg,\dgp,n)  =0\quad \textnormal{and}\quad\test^{\indep}_{\eps,\delta}(\alg,\dgp,n) = 0.
\]
This also means that if a test $\widehat{\test}_{\eps,\delta}$ has assumption-free validity as a test of $\test^*_{\eps,\delta}$ or $\test^{\indep}_{\eps,\delta}$, then it automatically has assumption-free validity as a test of $\test^{\textnormal{cpl}}_{\eps,\delta}$. Thus, the result~\eqref{eqn:power_optimal_coupling} on power for the coupled case, applies to a strictly broader class of tests $\widehat{\test}_{\eps,\delta}$, as compared to the result~\eqref{eqn:power_optimal} (from  Theorem~\ref{thm:optimality}, for the case $\xi=\xi'$), or the result~\eqref{eqn:power_optimal_independent} (for the case $\xi\indep\xi'$). However, because $\delta^{\textnormal{cpl}}_{\eps}$ from Theorem~\ref{thm:optimality_general} is the smallest probability over all possible couplings of $(\xi,\xi')$, we therefore have $\delta^{\textnormal{cpl}}_{\eps}\leq \delta^*_\eps$ and $\delta^{\textnormal{cpl}}_{\eps}\leq\delta^{\indep}_\eps$, and thus the upper bound on power given in~\eqref{eqn:power_optimal_coupling} can be larger than the upper bound in~\eqref{eqn:power_optimal} (from Theorem~\ref{thm:optimality}, for the case $\xi = \xi'$) or in~\eqref{eqn:power_optimal_independent} (for the case $\xi \indep \xi'$). In this sense, the result~\eqref{eqn:power_optimal_coupling} is neither strictly stronger nor strictly weaker than either of the other two results.
\end{document}